\documentclass[sigconf, nonacm]{acmart}

\newcommand\vldbauthors{\authors}
\newcommand\vldbtitle{\shorttitle} 
\newcommand\vldbpagestyle{plain} 
\usepackage{multirow}
\usepackage{booktabs}
\usepackage{graphicx}
\usepackage{stfloats}
\usepackage{multicol}
\usepackage{bm}
\usepackage{enumitem}
\usepackage{makecell}
\usepackage[ruled,linesnumbered,vlined]{algorithm2e}
\usepackage{amsthm}        

\newtheorem{theorem}{Theorem}

\newtheorem{definition}{Definition}

\begin{document}
\title{\texttt{FEAT}: A Linear-Complexity Foundation Model for Extremely Large Structured Data}

\author{
  Zhenghang Song$^1$, Tang Qian$^2$, Lu Chen$^{1*}$, Yushuai Li$^{3*}$, Zhengke Hu$^2$, \\
  Bingbing Fang$^2$, Yumeng Song$^3$, Junbo Zhao$^{2*}$, Sheng Zhang$^2$, Tianyi Li$^3$
}
\affiliation{%
  \institution{
    $^1$Zhejiang University \quad
    $^2$Ant Group \quad
    $^3$Aalborg University
  }
  \country{}
}
\email{ {22451060, luchen}@zju.edu.cn, {yusli, yumengs, tianyi}@cs.aau.dk }
\email{ {qiantang.qt, zhengke.hzk, jianjiu.fbb, zhaojunbo.zjb, sheng.zs}@antgroup.com }
%%
%% The "author" command and its associated commands are used to define the authors and their affiliations.
 %\author{ a}
% \affiliation{%
% \vspace{1cm}
%   \institution{}
   %\city{Hangzhou}
 %  \country{}
% }

% \author{
%     Zhenghang Song \\
%     Zhejiang University \\
%     \texttt{22451060@zju.edu.cn} \\
%     \And
%     Tang Qian \\
%     Ant Group \\
%     \texttt{qiantang.qt@antgroup.com} \\
%     \And
%     Lu Chen* \\
%     Zhejiang University \\
%     \texttt{luchen@zju.edu.cn} \\
%     \And
%     Yushuai Li* \\
%     Aalborg University \\
%     \texttt{yusli@cs.aau.dk} \\
%     \And
%     Zhengke Hu \\
%     Ant Group \\
%     \texttt{zhengke.hzk@antgroup.com} \\
%     \And
%     Bingbing Fang \\
%     Ant Group \\
%     \texttt{jianjiu.fbb@antgroup.com} \\
%     \And
%     Yumeng Song \\
%     Aalborg University \\
%     \texttt{yumengs@cs.aau.dk} \\
%     \And
%     Junbo Zhao \\
%     Ant Group \\
%     \texttt{zhaojunbo.zjb@antgroup.com} \\
%     \And
%     Sheng Zhang \\
%     Ant Group \\
%     \texttt{sheng.zs@antgroup.com} \\
%     \And
%     Tianyi Li \\
%     Aalborg University \\
%     \texttt{tianyi@cs.aau.dk} \\
% }

\begin{abstract}

Structured data is widely used in domains such as healthcare, finance, and scientific data management. 
Recent studies on structured data foundation models (SFMs) aim to support data analysis and mining tasks over such data, but still face scalability and generalization challenges when applied to real-world enterprise databases.
First, many SFMs rely on full self-attention, which introduces an $\mathcal{O}(N^2)$ computational bottleneck and limits the number of tuples that can be processed jointly. 
Second, directly replacing attention with linear-complexity sequence models may conflict with the permutation-invariant nature of structured data, introducing artificial order bias and degrading representation quality. 
Moreover, models trained only on synthetic data may struggle to generalize to the heavy-tailed and heterogeneous distributions commonly found in real-world databases.
To address these challenges, we propose \texttt{FEAT}, a linear-complexity foundation model for extremely large structured data. 
\texttt{FEAT} replaces quadratic attention with a multi-layer dual-axis encoding architecture. 
It integrates an adaptive-fusion bidirectional state-space model (AFBM) with convolutional gated linear attention (Conv-GLA), enabling cross-tuple contextualization in $\mathcal{O}(N)$ time while supporting permutation-invariant representation learning. 
To improve robustness under real-world data skewness, \texttt{FEAT} further adopts a hybrid structural causal pre-training pipeline with a robust reconstruction objective. 
Experiments on 12 real-world database benchmarks show that \texttt{FEAT} consistently outperforms representative SFMs on zero-shot tasks and scales linearly with structured-data sample length, achieving up to 50$\times$ faster inference latency.

\end{abstract}

\maketitle
\pagestyle{\vldbpagestyle}
\begingroup\small\noindent\raggedright\textbf{Reference Format:}\\
\vldbauthors. \vldbtitle.  
\endgroup
\begingroup





\section{Introduction}

%Structured data forms the bedrock of modern science and industry, housing critical information from genomics to global finance \cite{hollmann2025accurate, borisov2022deep}. {\color{red} XX}

Structured data is information organized under a predefined data model or schema, typically represented as matrices where each row corresponds to a sample and each column represents a specific feature or attribute, and it is widely used in domains such as genomics, healthcare, finance, e-commerce, and scientific data management~\cite{hollmann2025accurate, borisov2022deep}.
Analyzing structured data enables the discovery of complex relational patterns, the prediction of future outcomes, and automated decision-making. It is important for applications such as disease diagnosis, risk assessment, algorithmic trading, and personalized recommendation systems across diverse database environments~\cite{narayan2022can, reis2024generalizable, shaikhha2023demonstration, lin2024gfs}.

%Most existing predictive modeling on structured data \cite{grinsztajn2022tree} relies on tree-based ensembles \cite{breiman2001random, chen2016xgboost, ke2017lightgbm, prokhorenkova2018catboost} or specialized deep networks \cite{arik2021tabnet, gorishniy2021revisiting, somepalli2021saint}.  {\color{red} XXX}

Most machine learning (ML) methods for data analysis and mining on structured data~\cite{grinsztajn2022tree} rely on tree-based ensembles \cite{breiman2001random, chen2016xgboost, ke2017lightgbm, prokhorenkova2018catboost} or specialized deep networks \cite{arik2021tabnet, gorishniy2021revisiting, somepalli2021saint}. They typically require task-specific training with extensive hyperparameter tuning and large labeled datasets, which limits their ability to generalize across datasets and tasks.
Foundation models (FMs)~\cite{bommasani2021opportunities}, based on large-scale pretraining on massive datasets, are designed to generalize across tasks. 
Motivated by the success of FMs in natural language processing and computer vision~\cite{brown2020language, touvron2023llama, dosovitskiy2021image}, applying FMs to structured data analysis has become an important research direction.

Existing structured data FMs (SFMs),  such as TabPFN~\cite{hollmann2025accurate} and LimiX \cite{zhang2025limix},  employ Transformer~\cite{vaswani2017attention} pre-trained on large synthetic structured datasets. They formulate structured data analysis and mining as an in-context learning (ICL) problem~\cite{qu2025tabicl}, enabling models to handle diverse tasks with few-shot inference without dataset-specific retraining.
While these studies demonstrate promising performance on small-scale datasets, they still face three major challenges when deployed on real-world datasets:

%Foundation models \cite{bommasani2021opportunities} are large-scale models trained on massive datasets and designed to generalize across tasks. They have gained strong attention because they significantly improved performance in natural language processing and computer vision \cite{brown2020language, touvron2023llama, dosovitskiy2021image}. Extending this paradigm to structured data has therefore become an important research direction. Foundation models for structured data, also called large structured-data models (SFMs), such as TabPFN \cite{hollmann2025accurate} and LimiX \cite{zhang2025limix}, apply Transformers \cite{vaswani2017attention} pre-trained on synthetic corpora. These models recast structured data prediction as an in-context learning task \cite{qu2025tabicl}, enabling strong few-shot inference without dataset-specific tuning.

% Despite these breakthroughs, a critical ``reality gap'' prohibits their deployment on real-world datasets. Optimized primarily for small-scale, synthetic environments, current SFMs hit a severe computational and architectural wall when forced to process massive, unsequenced, and heavy-tailed empirical matrices. To unlock the true potential of foundation models for real-world structured data, we must overcome three interconnected challenges:

%While existing studies demonstrate promising performance on  {\color{red} XXX}, they still face three major challenges:

\noindent\textbf{\textit{Challenge I: How to reduce the quadratic complexity of cross-sample modeling for massive structured datasets?}} 
To capture the macroscopic distribution of structured data, SFMs rely on full self-attention for cross-sample modeling. However, the inherent quadratic complexity ($\mathcal{O}(N^2)$) of self-attention severely restricts the context window, consistently triggering kernel-error or out-of-memory failures near just 50,000 samples \cite{qu2025tabicl, ma2024tabdpt, grinsztajn2025tabpfn}. While heuristic divide-and-conquer strategies can mitigate this at test time \cite{ye2025closer}, they cannot circumvent the fundamental computational bottleneck during pre-training. 
% Since industrial-grade structured datasets routinely exceed millions of records~\cite{thomas2024retrieval}, this persistent $\mathcal{O}(N^2)$ barrier strictly prohibits existing SFMs from ever observing the true global data distribution. 
Since industrial-grade structured datasets routinely exceed millions of records~\cite{thomas2024retrieval}, this persistent $\mathcal{O}(N^2)$ barrier prevents existing SFMs from capturing the global data distribution and limits their use in real-world applications.
% Therefore, a fundamentally new architectural mechanism is required to break the quadratic bottleneck and scale to massive sample sizes.

\noindent\textbf{\textit{Challenge II: How to maintain expressive representations under linear-complexity modeling on permutation-invariant structured data?}} 
% To bypass the $\mathcal{O}(N^2)$ bottleneck, an intuitive approach is to adopt linear-complexity ($\mathcal{O}(N)$) sequence models, such as State Space Models (SSMs) \cite{gu2024mamba} or Linear Attention \cite{katharopoulos2020transformers}. However, directly transplanting these architectures to structured data precipitously leads to representation collapse \cite{ahamed2024mambatab, koch2025state}. The fundamental disconnect lies in the data topology: standard linear models are inherently designed for sequential data driven by temporal locality and sequential dependencies. Conversely, structured datasets are strictly permutation-invariant---the order of records carries no semantic meaning. 
% When recurrent-like SSMs process unordered tabular samples, their fixed-size hidden states are forced to compress global, non-directional interactions into a rigid computational bottleneck. This inherently imposes a detrimental ``recency bias'' that arbitrarily prioritizes newly observed records while catastrophically forgetting earlier global context, severely degrading the model's expressive capacity. Conversely, while Linear Attention mechanisms avoid strict state compression and better preserve global routing, they critically lack the aggressive dynamic filtering required to suppress the heavy intrinsic noise of structured features, allowing variance to destructively accumulate over massive contexts \cite{yang2024gated}. Consequently, existing linear models are structurally inadequate.
To reduce the quadratic complexity of cross-sample modeling, existing studies typically replace full self-attention with linear-complexity sequence models, such as state space models (SSMs) \cite{gu2024mamba} or linear attention \cite{katharopoulos2020transformers}. These mechanisms process sequences in $\mathcal{O}(N)$ time by maintaining compressed state representations or kernelized attention accumulators.
However, directly applying such linear sequence mechanisms to structured data often leads to representation collapse \cite{ahamed2024mambatab, koch2025state}. Unlike sequential text, samples in structured datasets are permutation-invariant and typically contain substantial noise. Existing linear architectures face a structural trade-off when applied in isolation. SSMs perform effective local filtering but compress unordered global interactions into a fixed hidden state, leading to semantic information loss. In contrast, linear attention preserves global feature interactions but lacks strong local filtering, allowing noise variance to accumulate across long structured contexts \cite{yang2024gated}. 

% Consequently, a single homogeneous linear architecture cannot simultaneously suppress tabular noise and preserve global expressive capacity.

\noindent\textbf{\textit{Challenge III: How to ensure stable optimization under heavy-tailed real-world structured data?}} 
Most existing SFMs pre-train on synthetic datasets under strict i.i.d. assumptions, relying on static loss functions like standard MSE \cite{wang2022t, zhang2025limix, grinsztajn2025tabpfn}. However, real-world structured data is inherently heteroscedastic and heavy-tailed, populated by extreme outliers \cite{shwartz2022, borisov2022deep}. When exposed to this empirical variance, static optimization objectives become hypersensitive, frequently triggering gradient explosions. This fragility is severely exacerbated during large-scale, dynamic multi-task pre-training. As batch task proportions fluctuate wildly and targets are frequently missing, computing dense gradients with static weights inevitably leads to mathematical instabilities, such as division-by-zero errors and total optimization collapse. 
% Consequently, scaling SFMs to real-world environments strictly requires an outlier-robust, dynamically balanced optimization mechanism.

%To tackle these challenges, we propose \textbf{FEAT}, a linear-complexity ($\mathcal{O}(N)$) foundation model natively designed for massive, real-world structured data. 
%{\color{red}To address the first challenge, we propose multi-layer dual-axis encoding to capture XX and XX information and reduce the memory complexity (?). It consists of 3 XXX and 1 xxx. Mamba xxxx Conv-GLA....
%To address the second challenge, XXXX
%To address the third challenge, XXXX}

To tackle these challenges, we propose \texttt{FEAT}, a linear-complexity FM for massive, real-world structured data. 
%To address the first challenge, we design the sample-axis modeling stage using linear-complexity sequence operators, enabling cross-sample modeling with $\mathcal{O}(N)$ time and memory complexity. This design removes the quadratic dependency introduced by full self-attention and allows the model to scale to large structured datasets.
To address the first challenge, we design linear-complexity encoders for representation learning to propagate information across data samples. The model transmits contextual signals through fixed-size hidden states and attention accumulators, avoiding the construction of an $N \times N$ attention matrix and eliminating quadratic pairwise interactions.
% To address the first challenge, we redesign the sample-axis modeling stage using a sequence-based mathematical formulation that avoids constructing an $N \times N$ attention matrix. By processing samples iteratively through fixed-size hidden state representations and attention accumulators, the computational cost and memory footprint are bounded to scale linearly with the number of input records. This architectural shift naturally achieves $\mathcal{O}(N)$ complexity, allowing \texttt{FEAT} to ingest and process massive structured datasets.
To address the second challenge, we propose a multi-layer dual-axis encoding architecture that combines two complementary encoding layers: adaptive-fusion bi-Mamba-2 (AFBM), which captures dynamic local dependencies across samples, and convolutional gated linear attention (Conv-GLA), which maintains global interactions through explicit memory accumulation. 
% Specifically, AFBM models permutation-invariant contexts without imposing directional bias, enabling effective local dependency extraction on non-sequential structured data. Conv-GLA provides an explicit global memory that preserves long-range interactions across samples. 
Together, these two encoding layers prevent representation collapse in linear-complexity architectures and maintain expressive semantic representations for massive structured sequences.
To address the third challenge, we propose a stable pre-training pipeline for structured data. The framework combines scale-free synthetic structural causal models (SCMs) with real-world structured datasets and employs a numerically robust Huber-based reconstruction loss. This design improves robustness to heavy-tailed noise and outliers while encouraging the model to learn stable relational patterns in structured data.
Our contributions are summarized as follows:

% First, we replace the $\mathcal{O}(N^2)$ sample-axis attention with a novel heterogeneous linear architecture that interleaves \textbf{Adaptive Bidirectional Mamba-2} layers and \textbf{Convolutional Gated Linear Attention (Conv-GLA)} in a carefully configured 3:1 ratio. The significance of this design is that it uniquely resolves the representation collapse inherent to homogeneous linear models on permutation-invariant data. Secondly, the Mamba-2 blocks act as aggressive dynamic filters to discard heavy tabular noise, while the Conv-GLA layer provides infinite-capacity explicit memory to preserve global macroscopic feature interactions. This synergistic fusion mathematically guarantees $\mathcal{O}(N)$ scaling without sacrificing expressive fidelity. Thirdly, to achieve stable optimization across extreme empirical distributions, we fundamentally redesign the pretraining paradigm. We introduce a hybrid data pipeline blending scale-free Synthetic Causal Models (SCMs) with real-world datasets, optimized via a numerically robust Huber-based reconstruction loss and strict query-mask decoupling. This mechanism ensures that the model successfully resists outlier-induced gradient explosions and is explicitly coerced to learn true causal relational patterns rather than exploiting hypersensitive local statistics. 

\begin{itemize}[leftmargin=3mm,labelsep=2pt]

\vspace{-2mm}
    \item We propose \texttt{FEAT}, the first industrial-grade SFM built on a linear-complexity multi-layer dual-axis architecture, enabling cross-sample modeling with strictly $\mathcal{O}(N)$ complexity and scalable learning on massive structured datasets.
    
    \item We design two complementary encoding layers, AFBM and Conv-GLA, to preserve expressive representations under linear modeling. AFBM captures dynamic local dependencies across samples, while Conv-GLA maintains global interactions through explicit memory accumulation, enabling robust representation learning for permutation-invariant structured data.

    \item We pioneer a mixed real-and-synthetic pretraining strategy fortified by a Huber-based objective. This effectively neutralizes outlier-induced gradient explosions, guaranteeing stable convergence across heteroscedastic, heavy-tailed empirical distributions and bridging the simulation-to-reality gap.
    
    \item Extensive experiments across 12 real-world benchmarks demonstrate that \texttt{FEAT} delivers up to a 50$\times$ inference speedup at extreme context lengths compared to existing foundation models. \texttt{FEAT} achieves this massive computational scalability while maintaining zero-shot predictive parity with state-of-the-art baselines in both classification and regression tasks.
\end{itemize}

We review related work in Section~\ref{sec:related_work} and cover preliminaries in Section~\ref{sec:preliminaries}. Section~\ref{sec:problem_statement} introduces the objective of  \texttt{FEAT}. Section~\ref{sec:methodology} presents \texttt{FEAT}, while Section~\ref{sec:experiment} discusses experimental findings. Section~\ref{sec:conclusion} concludes the paper and outlines research directions.
\vspace{-2mm}
\section{Related Work}
\label{sec:related_work}

\subsection{ML for Structured Data Mining}

ML on structured data is fundamental for predictive analytics within data management systems~\cite{boehm2020systemds}. 
Existing studies primarily fall into two categories: tree-based methods~\cite{breiman2001random, chen2016xgboost, ke2017lightgbm, prokhorenkova2018catboost} and deep learning-based methods~\cite{arik2021tabnet, gorishniy2021revisiting, somepalli2021saint}. 
Tree-based methods address common challenges in structured data, such as heterogeneous feature types, missing values, and unnormalized distributions, through split-based modeling and feature-specific handling. 
For example, CatBoost~\cite{prokhorenkova2018catboost} introduces ordered target encoding to handle high-cardinality categorical features while reducing target leakage. 
Deep learning-based methods aim to learn richer representations and capture complex feature interactions. 
TabNet~\cite{arik2021tabnet} employs sequential attention for instance-wise feature selection, allowing the model to focus on relevant features for each sample. 
SAINT~\cite{somepalli2021saint} uses dual self-attention over both rows and columns to model complex relational contexts among samples and features.

However, both approaches follow a training-centric paradigm, where models are optimized separately for each dataset, requiring dataset-specific optimization, extensive hyperparameter tuning, and multiple iterative data scans~\cite{kumar2015model, schott2022lmfao}. 
Systems such as Alpine Meadow~\cite{shang2020alpine} and Cerebro~\cite{nakandala2020cerebro} further show that configuring optimal ML pipelines for new databases incurs substantial system-level search overhead. 
This limits scalability on large datasets and hinders task-agnostic deployment in data systems~\cite{kramer2020willump}. 
In contrast, \texttt{FEAT} leverages an ICL pre-training framework to capture transferable relational patterns, thereby eliminating dataset-specific training and enabling zero-shot inference on unseen datasets.

% \subsection{Scalable Modeling for Very Large Databases}
% As enterprise databases routinely scale to millions or billions of tuples, handling extremely large sample sizes has become a challenge. Traditional approaches address this massive scale via system-level parallelization, utilizing distributed computing frameworks or histogram-based gradient approximations \cite{ke2017lightgbm} to accelerate table scans. While In-Database ML systems \cite{boehm2020systemds} push computation close to the data to avoid costly data movement, they remain bound to the aforementioned dataset-specific training paradigm.

% More recently, Tabular Foundation Models (e.g., TabPFN \cite{hollmann2025accurate}, LimiX \cite{zhang2025limix}) have attempted to provide zero-shot inference capabilities. However, their reliance on sample-wise full attention imposes an $\mathcal{O}(N^2)$ memory and computational bottleneck. To apply these models to very large databases, practitioners are forced to rely on ad-hoc data reduction techniques—such as random subsampling, mini-batch chunking, or coreset selection. These heuristic workarounds severely truncate the global context, discarding the heavy-tailed empirical noise and rare anomalies inherently present in massive, real-world datasets \cite{ma2024tabdpt}. 

\subsection{Scalable Data Mining on Large Datasets}

Scalable data mining on large datasets aims to support efficient learning and analysis when the number of samples grows to millions or billions. 
A common line of work improves scalability through algorithmic and system-level optimizations. 
LightGBM~\cite{ke2017lightgbm} uses gradient-based sampling and histogram-based approximation to reduce training cost. 
JoinBoost~\cite{huang2023joinboost} trains models directly over normalized relational tables through SQL execution, avoiding expensive data materialization. 
Another line of work relies on data reduction, where methods such as CoreTab~\cite{razmadze2025datamap} construct tabular coresets to accelerate training while preserving key data properties. 
For SFMs, TabPFN~\cite{hollmann2025accurate} and LimiX~\cite{zhang2025limix} have shown strong predictive performance using ICL over structured datasets.

However, existing scalable mining techniques do not directly address the scalability challenges of SFMs. 
Algorithmic and system-level optimizations are mainly designed for task-specific training, while data reduction methods improve efficiency by reducing the number of tuples, which may discard global distributional information. 
Meanwhile, existing SFMs rely on sample-wise attention, leading to an $\mathcal{O}(N^2)$ computational cost with respect to the number of tuples. 
To handle large datasets, they often require sampling or data reduction, which can limit their ability to capture global data characteristics, especially rare patterns and anomalies~\cite{ma2024tabdpt}. 
\texttt{FEAT} replaces full attention with a linear architecture based on AFBM and Conv-GLA, reducing the sample-axis complexity to $\mathcal{O}(N)$ and enabling efficient processing of large relational datasets while preserving global contextual information across tuples.

\vspace{-1mm}

\subsection{Linear-Complexity Sequence Models}

Linear-complexity sequence models aim to reduce the cost of modeling long contexts by avoiding pairwise attention over all samples. 
Linear Transformers~\cite{katharopoulos2020transformers} approximate attention with kernel feature maps, reducing the complexity from $\mathcal{O}(N^2)$ to $\mathcal{O}(N)$. SSMs, such as Mamba~\cite{gu2024mamba}, further compress historical context into fixed-size hidden states through selective state updates.  MambaTab~\cite{ahamed2024mambatab} adapt SSMs to structured data, reducing the memory cost of processing large numbers of samples.

However, directly applying standard sequential models to structured data introduces a representation mismatch. 
Structured data is permutation-invariant, while standard SSMs use a causal and unidirectional bias designed for ordered sequences. 
This mismatch can weaken global dependency modeling and make representations sensitive to artificial row orders and noisy samples. 
\texttt{FEAT} integrates AFBM for bidirectional sample-axis modeling, reducing the causal bias of one-directional sequence processing.
\vspace{-1mm}

\section{Preliminaries}
\label{sec:preliminaries}

\subsection{In-Context Learning for Structured Data} 

\begin{definition}
\label{def:structured_data}
A \textbf{structured dataset} is typically instantiated as a data matrix $\mathbf{X}\in\mathbb{R}^{N\times D}$, 
where $N$ is the number of samples and $D$ is the number of features. 
Each row $\mathbf{x}_i\in\mathbb{R}^{D}$ represents the feature vector 
of the $i$-th sample.
\end{definition}

\begin{definition}
\label{def:structured_label}
Given a structured dataset $\mathbf{X}\in\mathbb{R}^{N\times D}$, 
its \textbf{label vector} is denoted as $\mathbf{y}\in\mathcal{Y}^{N}$, 
where $\mathcal{Y}$ represents the label space and each element 
$y_i \in \mathcal{Y}$ corresponds to the target label of sample $\mathbf{x}_i$.
\end{definition}

The label space $\mathcal{Y}$ depends on the specific task type. 
For example, in classification tasks, $\mathcal{Y}$ is a finite set of discrete class labels, 
while in regression tasks, $\mathcal{Y}$ is a continuous subset of $\mathbb{R}$.

\begin{definition}
\label{def:icl_structured}
Given a structured dataset $\mathbf{X}\in\mathbb{R}^{N\times D}$ and its label vector $\mathbf{y}\in\mathcal{Y}^{N}$, 
let $\mathbf{X}_{L}\subset\mathbf{X}$ denote the subset of samples with observed labels 
$\mathbf{y}_{L}\subset\mathbf{y}$, and let $\mathbf{X}_{U}=\mathbf{X}\setminus\mathbf{X}_{L}$ 
denote the subset of query samples with unobserved labels $\mathbf{y}_{U}$. 
\textbf{ICL for structured data} learns a foundation model parameterized by $\theta$ to predict the label distribution of 
$\mathbf{X}_{U}$ conditioned on $\mathbf{X}_{U}$, $\mathbf{X}_{L}$, and $\mathbf{y}_{L}$ by approximating the true posterior predictive distribution:
\begin{equation}
P_{\theta}(\mathbf{y}_{U}\mid \mathbf{X}_{U}, \mathbf{X}_{L}, \mathbf{y}_{L})
\approx
P(\mathbf{y}_{U}\mid \mathbf{X}_{U}, \mathbf{X}_{L}, \mathbf{y}_{L}),
\label{eq:icl_objective}
\end{equation}
where $P_{\theta}(\cdot)$ denotes the predictive distribution induced by the model.
\end{definition}

At inference time, the labeled samples $(\mathbf{X}_{L},\mathbf{y}_{L})$ serve as the in-context support set, 
while $\mathbf{X}_{U}$ represents the query set. 
Through the ICL mechanism, the model directly infers the label distribution of $\mathbf{X}_{U}$ based on the relational patterns observed in the support samples. Since the predictive model is pre-trained to approximate the universal posterior predictive distribution, it can dynamically adapt to novel structured prediction tasks (e.g., varying feature dimensions or target types) using only the provided context, enabling seamless few-shot task adaptation while requiring absolutely no parameter updates or dataset-specific fine-tuning.

\subsection{State Space Models} 

SSMs~\cite{gu2024mamba} are a class of linear-complexity sequence models that capture temporal or spatial dependencies through a continuous latent state representation. 
Let $t \in \mathbb{R}$ denote the continuous time (or sequence) index. Given a 1-D input signal $u_t \in \mathbb{R}$, an SSM generates an output $y_t \in \mathbb{R}$ by evolving a hidden state $\mathbf{h}_t \in \mathbb{R}^{d_{state}}$ over time, where $d_{state}$ denotes the internal state capacity. 
The state evolution is governed by a continuous-time linear ordinary differential equation (ODE):
\begin{equation}
\mathbf{h}'_t = \mathbf{A}\mathbf{h}_t + \mathbf{B}u_t, 
\quad
y_t = \mathbf{C}\mathbf{h}'_t,
\label{eq:ssm_continuous}
\end{equation}
where $\mathbf{A} \in \mathbb{R}^{d_{state} \times d_{state}}$ is the state transition matrix that dictates the system's internal dynamics, $\mathbf{B} \in \mathbb{R}^{d_{state} \times 1}$ projects the input signal into the high-dimensional latent space, and $\mathbf{C} \in \mathbb{R}^{1 \times d_{state}}$ maps the latent state back to the output space.

To enable efficient computation on discrete hardware, the continuous time ODE must be discretized using a timescale parameter $\Delta$ (typically via the zero-order hold assumption). This transforms the continuous dynamics into a discrete-time recurrent sequence:
\begin{equation}
\mathbf{h}_t = \mathbf{\bar{A}} \mathbf{h}_{t-1} + \mathbf{\bar{B}} u_t, 
\quad 
y_t = \mathbf{C}\mathbf{h}_t,
\label{eq:ssm_discrete}
\end{equation}
where $\mathbf{h}_t$ denotes the latent state at discrete step $t$, and $\mathbf{\bar{A}}$ and $\mathbf{\bar{B}}$ are the discretized matrices derived from $\mathbf{A}$, $\mathbf{B}$, and $\Delta$. 

This recurrent formulation allows SSMs to process sequences with $\mathcal{O}(N)$ linear time complexity and constant memory. However, standard SSMs based on Eq.~\ref{eq:ssm_discrete} are inherently unidirectional and causal, which presents a fundamental structural challenge when applied to permutation-invariant structured data.
\section{Problem Statement}
\label{sec:problem_statement}

% In this work, we aim to design a linear-complexity Tabular Foundation Model capable of performing zero-shot ICL on sequences of massive length. We formalize this problem by first defining the inference objective, followed by an analysis of the dual computational and representational bottlenecks that currently prohibit infinite-context scaling.

% \subsection{The In-Context Inference Objective}
% Given a context set of $N_{ctx}$ labeled tabular samples, $\mathcal{D}_{ctx} = \{(\mathbf{x}_i, y_i)\}_{i=1}^{N_{ctx}}$, and a set of $M$ query samples $\mathbf{X}_{q} = \{\mathbf{x}_{N_{ctx}+j}\}_{j=1}^M$, our fundamental objective is to approximate the true posterior distribution of the query labels $\mathbf{Y}_q$ conditioned on the context. Following the Bayesian learning paradigm \cite{hollmann2025accurate, zhang2025limix}, we parameterize this predictive distribution using a neural network $f_{\theta}$:
% \begin{equation}
%     P_{\theta}(\mathbf{Y}_q \mid \mathbf{X}_q, \mathcal{D}_{ctx}) \approx P(\mathbf{Y}_q \mid \mathbf{X}_q, \mathcal{D}_{ctx})
%     \label{eq:icl_objective}
% \end{equation}
% Crucially, under the ICL framework, the model parameters $\theta$ remain frozen during inference. The network must dynamically map the joint sequence of context and queries into an implicit representation space to deduce task-specific patterns and feature relationships without gradient updates.

\subsection{The Dual Bottlenecks of SFMs}

Existing SFMs~\cite{hollmann2025accurate, zhang2025limix,qu2025tabicl} follow the ICL paradigm defined in Def.~\ref{def:icl_structured}. 
Given a dataset $\mathbf{X}\in\mathbb{R}^{N\times D}$ and partially observed labels $\mathbf{y}$, the model takes labeled samples $(\mathbf{X}_{L},\mathbf{y}_{L})$ as context and predicts labels for $\mathbf{X}_{U}$. 
To support context-conditioned reasoning, each feature value in a sample $\mathbf{x}_i$ is combined with its column information and mapped to a $d$-dimensional embedding, transforming the original matrix into a tensor $\mathcal{X}\in\mathbb{R}^{N\times D\times d}$. 
Most models then apply full self-attention across the sample dimension to capture cross-sample interactions, which introduces a quadratic cost $\mathcal{O}(N^2 \cdot D \cdot d)$ and limits scalability to large datasets. 
Replacing attention with linear models such as SSMs reduces the complexity to $\mathcal{O}(N)$, but directly applying them to structured data leads to representation challenges.

\noindent\textbf{Bottleneck 1: The Linear Trap in Structured Data.} 
Linear sequence models summarize all historical information into a fixed-size hidden state $\mathbf{h}_i \in \mathbb{R}^{d_{state}}$. For long structured sequences with heterogeneous samples and high noise levels, this state representation forces the model to compress large amounts of information into a limited capacity. 
As the sequence grows, useful signals are gradually diluted by noise, leading to progressive information decay and a severe degradation of the signal-to-noise ratio. 

\noindent\textbf{Bottleneck 2: Causal Mask Deficit in Structured Data.} 
Unlike tokens in natural language, samples in a structured dataset do not possess an intrinsic temporal ordering, and their representation should therefore be permutation-invariant. 
Sequential SSMs, however, process inputs in a strictly ordered manner and implicitly impose a lower-triangular causal mask over the sequence. 
When structured samples are arranged in an arbitrary order, this artificial causality restricts information flow and prevents early-indexed samples from accessing forward contextual information. 
As a result, the model cannot fully capture global relationships across samples, leading to biased and sub-optimal representations.

\subsection{Objective of \texttt{FEAT}}

% Our primary architectural challenge is to design a linear mapping function $f_{\theta}$ that simultaneously breaks the $\mathcal{O}(N^2)$ memory wall while fundamentally circumventing both the Linear Trap and the Causal Mask Deficit.

\texttt{FEAT} aims to learn a zero-shot in-context FM with linear complexity for massive structured sequences. 
Its objective is to learn a mapping function $f_{\theta}$ that predicts the label distribution of unlabeled samples from labeled in-context examples, while requiring no task-specific parameter updates during inference.
Different from existing Transformer-based SFMs~\cite{hollmann2025accurate, zhang2025limix}, \texttt{FEAT} is developed to break the $\mathcal{O}_{attn} = \mathcal{O}(N^2 \cdot D \cdot d)$ memory wall induced by full self-attention. 
Different from standard linear sequence models, \texttt{FEAT} is further designed to avoid two key bottlenecks in structured sequence modeling: excessive hidden-state compression, which limits long-range information preservation, and artificial unidirectional ordering, which violates the permutation-invariant nature of structured data.
In this way, \texttt{FEAT} seeks to provide a zero-shot ICL model that scales linearly with the number of samples, preserves global cross-sample dependencies, and supports multiple downstream tabular tasks.
\section{FEAT}
\label{sec:methodology}

\begin{figure*}[t]
  \centering
  % Note: Ensure you upload 'architecture.pdf' to a 'figure' folder in Overleaf, 
  % or replace the path with your actual image file.
  \includegraphics[width=1.0\textwidth]{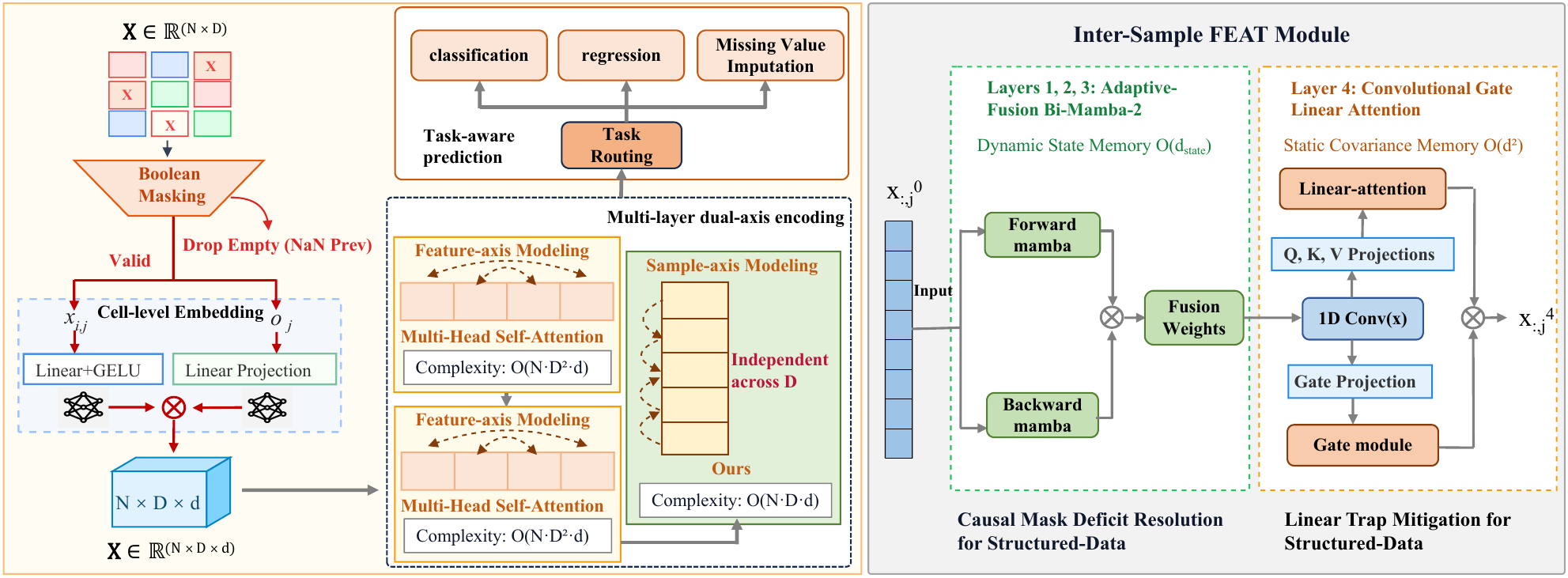}
  \vspace{-5mm}
  \caption{The overview of \texttt{FEAT}.}
  \label{fig:framework}
  \vspace{-5mm}
\end{figure*}

\subsection{Overview of \texttt{FEAT}}
In Fig.~\ref{fig:framework}, \texttt{FEAT} consists of three main components: cell-level embedding, multi-layer dual-axis encoding, and task-aware prediction.

\noindent \textit{\textbf{Cell-level embedding.}} 
Given a structured dataset $\mathbf{X}\in\mathbb{R}^{N\times D}$, \texttt{FEAT} first projects each data cell into a dense embedding space and converts the raw input into a 3D tensor $\mathcal{X}^{(0)} \in \mathbb{R}^{N \times D \times d}$, where $N$ is the number of samples, $D$ is the number of features, and $d$ is the embedding dimension. 
This representation preserves both the sample-wise and feature-wise organization of structured data, serving as the foundation for subsequent dual-axis representation learning. 
The detailed embedding process is described in Section~\ref{sec:embedding}.

\noindent \textit{\textbf{Multi-layer dual-axis encoding.}} 
The tensor $\mathcal{X}^{(0)}$ is processed by $L$ stacked \texttt{FEAT} blocks, each of which decomposes structured-data representation learning into two sequential modeling stages: \textit{Feature-axis modeling} and \textit{Sample-axis modeling}. 
The feature-axis modeling stage models intra-sample feature dependencies by applying self-attention across the $D$ feature columns of each sample, efficiently capturing local semantic correlations. 
The sample-axis modeling stage models inter-sample dependencies along the sample dimension using a linear-complexity mechanism, enabling efficient long-range cross-sample interaction modeling. 
The architecture of the dual-axis encoding blocks is presented in Section~\ref{sec:feat_block}.

\noindent \textit{\textbf{Task-aware prediction.}} 
After passing through the stacked \texttt{FEAT} blocks, the resulting contextualized representations are used to predict the labels of unlabeled samples conditioned on the labeled in-context examples. 
This design enables \texttt{FEAT} to perform zero-shot ICL across multiple structured-data prediction tasks, including classification, regression and missing value imputation. 
The task prediction mechanism is described in Section~\ref{sec:prediction}.

\noindent \textit{\textbf{Pre-training.}}
To enable zero-shot generalization across diverse structured-data tasks, \texttt{FEAT} is pre-trained on a large corpus of synthetic structured datasets generated by an advanced hybrid SCM pipeline. 
During training, \texttt{FEAT} follows a context-conditional masked modeling paradigm and jointly optimizes feature reconstruction and task prediction objectives. 
The detailed pre-training procedure is described in Section~\ref{sec:pretraining}.

% \subsection{Architecture}

% \subsubsection{Overall Architecture of the Janus Framework}
% \label{subsec:overall_arch}

% Janus projects the raw sequence matrix via cell-level embedding into the dense 3D tensor $\mathcal{X}^{(0)} \in \mathbb{R}^{N \times D \times d}$. To efficiently capture both fine-grained feature dependencies and global sequence contexts, Janus routes this tensor through a bidirectional heterogeneous dual-axis architecture. 

% The backbone consists of $L$ sequential hybrid blocks. Each block strictly decouples the spatial routing into two alternating phases:

% \textbf{1. Feature-Axis Routing (Intra-Sample):} For a given sample $i$, standard multi-head self-attention is applied across its $D$ feature columns. Since $D$ is typically small and fixed, this $\mathcal{O}(N \cdot D^2 \cdot d)$ operation safely and efficiently extracts local semantic correlations.
    
% \textbf{2. Sample-Axis Routing (Inter-Sample Janus Block):} To break the $\mathcal{O}(N^2)$ barrier across samples, the tensor is transposed, and each feature column $j$ is routed independently through our \textbf{Janus module}. The Janus module fundamentally redesigns sequence modeling by sequentially bridging bounded dynamic state modeling (Bi-Mamba-2) with infinite-capacity static memory (Conv-GLA).

\subsection{Cell-level Embedding}~\label{sec:embedding}
Unlike natural language, structured data $\mathbf{X}\in\mathbb{R}^{N\times D}$ is characterized by strict permutation invariance along both the sample axis and the feature axis. Flattening such data into 1D sequences intrinsically corrupts this structural independence. To process this natively, \texttt{FEAT} transforms raw data cells $x_{i,j}$ into a 3D contextual tensor $\mathcal{X}^{(0)} \in \mathbb{R}^{N \times D \times d}$ using two independent, additive embeddings: a value projection and a column-axis feature identifier.

% \noindent\textbf{Value Projection.}
First, we uniquely map continuous scalar values to a $d$-dimensional semantic space. The structured data is encoded via a non-linear MLP, while missing or intentionally masked values are substituted with a shared learnable token. Formally, the embedding $\mathbf{e}^{val}_{i,j}$ for the cell at row $i$ and column $j$ in $\mathbf{X}$ is computed as:
\begin{equation}
    \mathbf{e}^{val}_{i,j} = 
    \begin{cases}
        \text{MLP}(x_{i,j}), & \text{if } x_{i,j} \text{ is observed}, \\
        \mathbf{m}_{mask}, & \text{if } x_{i,j} \text{ is missing or masked}.
    \end{cases}
\end{equation}
Where $MLP(\cdot)$ is the linear MLP model and  $\mathbf{m}_{mask} \in \mathbb{R}^{d}$ is the shared learnable token.

% \noindent\textbf{Column-Axis Discriminative Feature Encoding (DFE).}
Then, we apply standard positional encodings across the feature dimension to impose a false geometric distance bias. To strictly preserve column permutation invariance, \texttt{FEAT} introduces subspace orthogonal discriminative feature encoding (S-DFE). Instead of learning static positional mappings, S-DFE dynamically provides equidistant feature identities.

In each forward pass, we randomly sample a low-rank orthogonal matrix $\mathbf{O} \in \mathbb{R}^{D \times \frac{d}{4}}$, where each row vector $\mathbf{o}_j \in \mathbb{R}^{\frac{d}{4}}$ serves as the base identity for the $j$-th column. The pairwise orthogonality ($\mathbf{o}_j^\top \mathbf{o}_k = 0$ for $j \neq k$) guarantees that every feature is distinctly identifiable yet equidistant in the subspace, eliminating associative ordering bias. This identity signature is linearly projected to the hidden dimension $d$ via a learnable weight matrix $\mathbf{W}_{dfe} \in \mathbb{R}^{\frac{d}{4} \times d}$:
\begin{equation}
    \mathbf{loc}^{col}_j = \mathbf{o}_{j}\mathbf{W}_{dfe} 
\end{equation}
This dynamic strategy allows \texttt{FEAT} to effectively model columns purely based on their distinct identities, achieving inherent zero-shot adaptability on unseen structured datasets with arbitrary schemas.
The final cell representation $\mathcal{X}^{(0)}_{i,j}$ is obtained by element-wise aggregating the data projection and the column-axis identifier, followed by layer normalization for stable optimization:
\begin{equation}
    \mathcal{X}^{(0)}_{i,j} = \text{LayerNorm} \left( \mathbf{e}^{val}_{i,j}  + \mathbf{loc}^{col}_j \right)
\end{equation}
This fusion yields a 3D matrix $\mathcal{X}^{(0)}$, which serves as the foundation for the dual-axis modeling blocks.

\subsection{Multi-layer dual-axis encoding}~\label{sec:feat_block}
Following the initial embedding stage, the 3D tensor $\mathcal{X}^{(0)} \in \mathbb{R}^{N \times D \times d}$ is processed by $L$ stacked dual-axis encoding blocks. Each encoding block consists of two feature axis modeling and one sample axis modeling. To effectively capture both local attribute correlations and global context without suffering from computationally prohibitive dense attention over the flattened $N \times D$ sequence, each block $l \in \{1, \dots, L\}$ factorizes the representation learning into two orthogonal stages: Feature-axis modeling and Sample-axis modeling.

\subsubsection{Feature-axis Modeling}

The primary objective of the feature-axis modeling stage is to model the intra-sample dependencies (i.e., interactions among different features within the same sample). Because features in structured data often exhibit complex, non-linear correlations, we employ multi-head self-attention (MHSA) across the feature dimension $D$.

For the $l$-th block, given the input tensor $\mathcal{X}^{(l-1)} \in \mathbb{R}^{N \times D \times d}$, we strictly isolate the computations across the sample axis ($N$). For each independent sample index $i \in \{1, \dots, N\}$, we extract its corresponding 2D feature matrix $\mathbf{F}_i^{(l-1)} = \mathcal{X}^{(l-1)}_{i, :, :} \in \mathbb{R}^{D \times d}$. The routing is then defined by standard Transformer layer operations—comprising MHSA and a Feed-Forward Network (FFN), interleaved with Layer Normalization (LN) and residual connections:
\begin{align}
    \tilde{\mathbf{F}}_i^{(l-1)} &= \mathbf{F}_i^{(l-1)} + \text{MHSA}\Big(\text{LN}\big(\mathbf{F}_i^{(l-1)}\big)\Big), \\
    \hat{\mathbf{F}}_i^{(l-1)} &= \tilde{\mathbf{F}}_i^{(l-1)} + \text{FFN}\Big(\text{LN}\big(\tilde{\mathbf{F}}_i^{(l-1)}\big)\Big).
\end{align}
The operation is strictly constrained within the $D$ tokens of sample $i$. The MHSA mechanism aggregates information from all $D$ columns to contextualize each specific feature representation, seamlessly updating the value embeddings based on the S-DFE identities introduced in the previous stage. 

Finally, the independently processed matrices $\hat{\mathbf{F}}_i^{(l-1)}$ for all $N$ samples are stacked back together to form the intermediate 3D tensor $\hat{\mathcal{X}}^{(l-1)} \in \mathbb{R}^{N \times D \times d}$, which serves as the direct input to the Sample-axis routing stage. 

\subsubsection{Sample-Axis Modeling}

To simultaneously address the \textit{linear trap} and the \textit{causal mask deficit} while maintaining linear computational complexity, \texttt{FEAT} introduces a heterogeneous sequence modeling architecture along the sample dimension. 
Instead of relying on a single sequential module, \texttt{FEAT} integrates two complementary mechanisms: AFBM which captures dynamic local dependencies across samples, and Conv-GLA which provides stable global memory over arbitrarily long sequences.

Specifically, the sample-axis modeling stage consists of four continuous layers arranged in a fixed topology: three AFBM layers ($l \in \{1,2,3\}$) followed by one Conv-GLA layer ($l = 4$). 
The AFBM layers perform efficient bidirectional sequence modeling to extract local inter-sample dynamics while mitigating the causal mask deficit of standard SSMs. 
However, since state-space models inherently compress historical information into bounded hidden states, they remain susceptible to the linear trap when the sequence length becomes extremely large. 
To alleviate this limitation, the Conv-GLA layer introduces an explicit global memory reservoir that stores long-range dependencies without forcing excessive compression into the hidden state.

\noindent\textbf{AFBM.}  For an input representation $\hat{\mathcal{X}}_{i,j}^{(l-1)} \in \mathbb{R}^{d}$, the model computes forward and backward state transitions:
\begin{align}
    \overrightarrow{\mathbf{h}}_{i,j}^{(l)} &= \mathbf{\bar{A}}^{(l)} \overrightarrow{\mathbf{h}}_{i-1,j}^{(l)} + \mathbf{\bar{B}}^{(l)} \hat{\mathcal{X}}_{i,j}^{(l-1)}, \\
    \overleftarrow{\mathbf{h}}_{i,j}^{(l)} &= \text{Mamba-Bwd}\!\left(\hat{\mathcal{X}}_{i,j}^{(l-1)}\right),
\end{align}
where $\overrightarrow{\mathbf{h}}_{i,j}^{(l)}, \overleftarrow{\mathbf{h}}_{i,j}^{(l)} \in \mathbb{R}^{d_{state}}$ represent the forward and backward hidden states, respectively, for the $j$-th feature of the $i$-th sample at the $l$-th layer. $\mathbf{\bar{A}}^{(l)} \in \mathbb{R}^{d_{state} \times d_{state}}$ and $\mathbf{\bar{B}}^{(l)} \in \mathbb{R}^{d_{state} \times d}$ denote the discretized state transition and input projection matrices derived from the continuous SSM dynamics. $\text{Mamba-Bwd}(\cdot)$ indicates the identical Mamba-2 state-space operation computed in the reverse sequence order (from sample $i=N$ down to $1$).

The two directional states are fused via a learnable projection:
\begin{equation}
    \mathcal{X}_{i,j}^{(l)} =
    \mathbf{W}_{fuse}^{(l)}
    \left[
    \overrightarrow{\mathbf{h}}_{i,j}^{(l)}
    \parallel
    \overleftarrow{\mathbf{h}}_{i,j}^{(l)}
    \right],
\end{equation}
where $\mathcal{X}_{i,j}^{(l)} \in \mathbb{R}^d$ is the fused contextual representation. The operator $[\cdot \parallel \cdot]$ denotes vector concatenation along the hidden state dimension, yielding a bidirectional state in $\mathbb{R}^{2d_{state}}$. $\mathbf{W}_{fuse}^{(l)} \in \mathbb{R}^{d \times 2d_{state}}$ is the learnable weight matrix responsible for projecting the concatenated state back to the original embedding dimension $d$.

This bidirectional fusion removes the artificial causality imposed by standard SSMs and enables permutation-aware feature extraction across samples. 
Nevertheless, since AFBM still relies on bounded hidden states of size $\mathcal{O}(d_{state})$, its capacity for storing long-range information remains limited when the sequence length grows very large.

\noindent\textbf{Conv-GLA.} Unlike standard linear attention that indiscriminately aggregates all samples and amplifies noise, Conv-GLA first applies localized smoothing before constructing global memory.
Specifically, a 1D convolution acts as a low-pass filter over the output of the third AFBM layer:
\begin{equation}
    \tilde{\mathcal{X}}_{i,j}^{(3)} =
    \text{Conv1D}\!\left(
    \mathcal{X}_{i,j}^{(3)}, \text{kernel\_size}=K
    \right),
\end{equation}
where $\tilde{\mathcal{X}}_{i,j}^{(3)} \in \mathbb{R}^d$ is the smoothed feature representation, $\mathcal{X}_{i,j}^{(3)}$ is the output from the final ($l=3$) AFBM layer, and $K$ denotes the kernel size of the 1D depthwise convolution applied along the sample dimension $N$ to aggregate local context and attenuate high-frequency noise.

We then obtain the key and value vectors via linear projections: $\mathbf{k}_{i,j} = \mathbf{W}_k \tilde{\mathcal{X}}_{i,j}^{(3)}$ and $\mathbf{v}_{i,j} = \mathbf{W}_v \tilde{\mathcal{X}}_{i,j}^{(3)}$, where $\mathbf{W}_k, \mathbf{W}_v \in \mathbb{R}^{d \times d}$. A data-dependent gating mechanism modulates the value vectors:
\begin{equation}
    \mathbf{g}_{i,j} =
    \text{SiLU}(\mathbf{W}_g \tilde{\mathcal{X}}_{i,j}^{(3)}),
    \quad
    \tilde{\mathbf{v}}_{i,j} =
    \mathbf{g}_{i,j} \odot \mathbf{v}_{i,j},
\end{equation}
where $\mathbf{g}_{i,j} \in (0,1)^d$ is the dynamic gating vector that evaluates the informativeness of the current sample. $\mathbf{W}_g \in \mathbb{R}^{d \times d}$ is the learnable gating projection matrix, $\text{SiLU}(\cdot)$ is the sigmoid linear unit activation function used to ensure smooth, non-linear gating probabilities, and $\odot$ denotes the element-wise (Hadamard) product.
The gated values are accumulated into a covariance memory matrix:
\begin{equation}
    \mathbf{S}_{i,j} =
    \mathbf{S}_{i-1,j} +
    \phi(\mathbf{k}_{i,j})
    \left(
    \mathbf{g}_{i,j} \odot \mathbf{v}_{i,j}
    \right)^\top,
\end{equation}
where $\mathbf{S}_{i,j} \in \mathbb{R}^{d \times d}$ is the accumulated global covariance memory matrix at step $i$, summarizing all historical key-value interactions for the $j$-th feature. $\phi(\cdot)$ is a non-linear kernel mapping function applied to the key vectors to maintain positive semi-definiteness, and $\top$ denotes the vector transpose operation forming the outer product between the mapped key and the gated value.

By adaptively suppressing uninformative samples ($\mathbf{g}_{i,j}\rightarrow 0$), the variance of the accumulated memory remains bounded, ensuring stable global context modeling even for extremely long sequences.

This carefully designed heterogeneous topology explicitly decouples aggressive local noise filtering (handled by the dynamic state memory of AFBM) from global long-range storage (handled by the static covariance memory of Conv-GLA). 
As a result, \texttt{FEAT} achieves scalable long-context modeling while avoiding both representation collapse and sequence-length degradation.

\subsection{Task-aware prediction}~\label{sec:prediction}
After the dual-axis representation learning across the $L$ stacked \texttt{FEAT} blocks, the network outputs a final contextualized 3D tensor $\mathcal{X}^{(L)} \in \mathbb{R}^{N \times (D+1) \times d}$. The sequence dimension expands to $D+1$ because a unified label token, acting as the target variable prompt, is appended to the $D$ feature tokens of each sample. For the labeled context samples in $\mathbf{X}_L$, this token encodes the ground-truth label $\mathbf{y}_L$; For the unlabeled query samples in $\mathbf{X}_U$, it is initialized as a learnable mask embedding.

To support downstream zero-shot prediction and self-supervised pre-training simultaneously, \texttt{FEAT} decouples the decoding process into three task-specific heads. We first extract the task-relevant hidden states from $\mathcal{X}^{(L)}$ for the query samples $i \in \{N_L + 1, \dots, N\}$, where $N_L = |\mathbf{X}_L|$ is the number of in-context examples:
\begin{align}
    \text{Target state:} \quad & \mathbf{z}^{y}_i = \mathcal{X}^{(L)}_{i, D+1, :} \in \mathbb{R}^{d} \\
    \text{Feature states:} \quad & \mathbf{z}^{x}_{i,j} = \mathcal{X}^{(L)}_{i, j, :} \in \mathbb{R}^{d}, \quad \forall j \in \{1, \dots, D\}
\end{align}

The target state $\mathbf{z}^{y}_i$ fundamentally aggregates both intra-sample feature dependencies and inter-sample supervision signals, while the feature states $\mathbf{z}^{x}_{i,j}$ capture the contextualized semantic representation of individual structured attributes. \texttt{FEAT} dynamically routes these states to the corresponding prediction heads, utilizing parallel MLPs with tailored architectures. For all following MLPs, we denote $d_{hidden}$ as the intermediate hidden projection dimension, and bias terms are omitted for notational brevity:

\noindent\textbf{Classification head.} For classification tasks with a finite label space of size $C$, the target state $\mathbf{z}^{y}_i$ is mapped to class logits via a two-layer network with a { GELU activation~\cite{GELU}}:
\begin{equation}
    \hat{\mathbf{y}}_i = \mathbf{W}_{c2}\big(\text{GELU}(\mathbf{W}_{c1}\mathbf{z}^{y}_i)\big) \in \mathbb{R}^{C},
\end{equation}
where $\mathbf{W}_{c1} \in \mathbb{R}^{d_{hidden} \times d}$ and $\mathbf{W}_{c2} \in \mathbb{R}^{C \times d_{hidden}}$ are the two learnable linear projection weights. The predicted categorical distribution closely approximates the true posterior $P_{\theta}(\mathbf{y}_{U}\mid \mathbf{X}_{U}, \mathbf{X}_{L}, \mathbf{y}_{L})$ as defined in Eq.~\ref{eq:icl_objective} via standard Softmax normalization.

\noindent\textbf{Regression head.} For continuous regression tasks, predicting unscaled scalar values is susceptible to explosive variance. Thus, we integrate Layer Normalization within the hidden projection to bound the activation magnitude before computing the final scalar output:
\begin{equation}
    \hat{y}_i = \mathbf{W}_{r2}\Big(\text{GELU}\big(\text{LayerNorm}(\mathbf{W}_{r1}\mathbf{z}^{y}_i)\big)\Big) \in \mathbb{R},
\end{equation}
where $\mathbf{W}_{r1} \in \mathbb{R}^{d_{hidden} \times d}$ and $\mathbf{W}_{r2} \in \mathbb{R}^{1 \times d_{hidden}}$ parameterize the regression MLP to output a continuous scalar.

\noindent\textbf{Missing value imputation head (MVIH).} To serve the context-conditional masked modeling (CCMM) pre-training objective and perform missing data imputation dynamically, \texttt{FEAT} must reconstruct the masked features. The feature state $\mathbf{z}^{x}_{i,j}$ is transformed back into the original numerical feature space using a stabilized architecture homologous to the regression head:
\begin{equation}
    \hat{x}_{i,j} = \mathbf{W}_{f2}\Big(\text{GELU}\big(\text{LayerNorm}(\mathbf{W}_{f1}\mathbf{z}^{x}_{i,j})\big)\Big) \in \mathbb{R},
\end{equation}
where $\mathbf{W}_{f1} \in \mathbb{R}^{d_{hidden} \times d}$ and $\mathbf{W}_{f2} \in \mathbb{R}^{1 \times d_{hidden}}$ represent the projection matrices specifically tailored for the task of accurate instance-level feature reconstruction.

This decoupled multi-head decoding strategy enables \texttt{FEAT} to perform seamless zero-shot downstream inference while simultaneously providing robust self-supervised feature reconstruction. By relying purely on the unified structural representations conditioned on the implicitly retrieved rules from the labeled context $\mathbf{X}_L$, the model circumvents any requirement for task-specific parameter updates or architecture modifications at inference time.

\subsection{Pre-training}~\label{sec:pretraining}
\vspace{-1mm}
\subsubsection{Hybrid SCM Generation Pipeline}
\label{subsec:pretraining_scm}

\vspace{-2mm}
While existing SFM pre-train on synthetic SCM, their hierarchical Directed Acyclic Graphs (DAGs) typically generate homogeneous structures with independent, identically distributed (i.i.d.) root variables. Such assumptions fail to reflect the statistical complexities of real-world structured-data domains. To synthesize training corpora $\mathbf{X} \in \mathbb{R}^{N \times D}$ and labels $\mathbf{y}$ that better match the dual-axis modeling capacity of \texttt{FEAT}, we introduce an advanced hybrid SCM generation pipeline. This pipeline progressively incorporates realistic structural variations and non-stationary statistical characteristics through four steps.

\begin{figure*}[t]
  \centering
  % Note: Ensure you upload 'architecture.pdf' to a 'figure' folder in Overleaf, 
  % or replace the path with your actual image file.
  \vspace{-5mm}
  \includegraphics[width=1.0\textwidth]{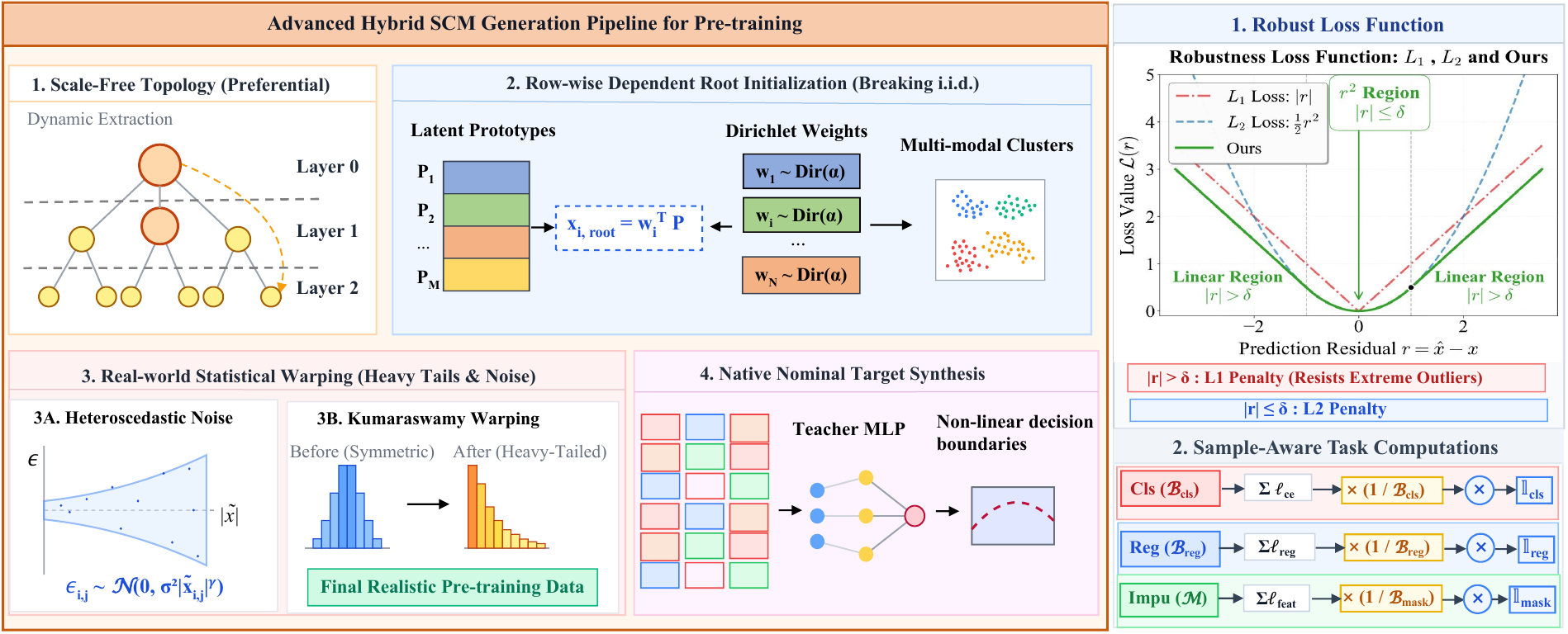}
  \vspace{-5mm}
  \caption{Pre-training methods of \texttt{FEAT}.}
  \label{fig:pre-train}
  \vspace{-5mm}
\end{figure*}

\noindent\textbf{(i) Scale-free causal graph generation.}
Standard generators use uniform edge sampling, ignoring feature importance. Instead, we construct the DAG topology $G = (V, E)$ over $D$ nodes using a preferential attachment mechanism. This yields a scale-free topology, allowing certain variables to emerge as highly connected "hubs." This design directly mimics the behavior of universal confounders in real-world datasets (e.g., "Age" or "Income"), which intrinsically dictate the distribution of numerous downstream variables.

\noindent\textbf{(ii) Prototype-based root initialization.}
Let $V_{root} \subset V$ denote the set of $D_{root}$ root nodes (in-degree of 0). To break the conventional i.i.d. assumption and explicitly introduce multi-modal row correlations, we assume the existence of $M$ latent prototypes $\mathbf{P} \in \mathbb{R}^{M \times D_{root}}$ representing underlying population clusters. For each sample $i \in \{1,\dots,N\}$, we draw an individual mixture weight vector from a Dirichlet distribution $\mathbf{w}_i \sim \text{Dir}(\bm{\alpha})$. The root features are then synthesized as:
\begin{equation}
    \mathbf{x}_{i, root} = \mathbf{w}_i^\top \mathbf{P} \in \mathbb{R}^{D_{root}}
\end{equation}
This explicitly injects sample-wise clustering structures commonly observed in structured patient/user profiles.

\noindent\textbf{(iii)  Causal propagation with heteroscedastic noise.}
For every continuous non-root node $j \in V \setminus V_{root}$, its intermediate value $\tilde{x}_{i,j}$ is propagated via a non-linear causal mechanism $f_j$ acting on its parent variables. We denote this set of parent variables as $\text{Pa}(j)$. Specifically, the intermediate representation is computed as $\tilde{x}_{i,j} = f_j(\mathbf{x}_{i, \text{Pa}(j)})$. To simulate measurement flaws and value-dependent uncertainties, we deliberately choose to abandon static Gaussian noise and introduce exponential heteroscedastic noise. The final variable value is modified as:
\begin{equation}
    x_{i,j} = \tilde{x}_{i,j} + \epsilon_{i,j}, \quad \text{where } \epsilon_{i,j} \sim \mathcal{N}(0, \sigma^2 |\tilde{x}_{i,j}|^\gamma),
\end{equation}
where $\gamma$ explicitly scales the noise variance with the magnitude of the signal, a hallmark of heavily skewed real-world metrics.

\noindent\textbf{(iv) Nominal target synthesis \& marginal warping.}
To generate the final target labels $y_i$, we eschew the naive ordinal binning used in prior works. Instead, we employ a native nominal synthesis process where a sparse teacher MLP queries the intermediate causal representations explicitly to compute target logits. By stochastically modulating the Signal-to-Noise Ratio in these logits, this procedure inherently produces highly non-linear, heterogeneous decision boundaries mapping to regression or classification tasks. 

Finally, to faithfully emulate heavy-tailed outliers, we apply stochastic Kumaraswamy warping to the continuous variables. This transformation skews the marginal feature distributions toward heavy-tailed domains while precisely preserving the monotonic ranking and dependencies implied by the causal structure. 

Together, these sequentially applied structural and statistical innovations synthesize a fundamentally rigorous pre-training corpus, substantially narrowing the reality gap between synthetic SCM arrays and naturally occurring structured-data distributions.

\vspace{-1mm}
\subsubsection{Loss Function}
\label{subsec:pretraining_objective}

Following the context-conditional masked modeling paradigm, the pre-training objective of \texttt{FEAT} jointly optimizes feature reconstruction and task prediction. To ensure numerical stability under heavy-tailed structured-data distributions and highly variable context lengths, the optimization pipeline incorporates robust measurement mechanisms and a dynamic loss balancing strategy.

For the self-supervised feature reconstruction objective, we adopt the Smooth L1 loss rather than the standard mean squared error (MSE). Structured data frequently contains extreme outliers manifesting from heavy-tailed distributions, which trigger exploding gradients when penalized squarely. The Huber mechanism acts as an L2 penalty for minor reconstruction errors while smoothly transitioning to an L1 penalty for severe deviations. Let $\mathcal{M}$ denote the set of masked cell coordinates $(i,j)$ in a given batch, where $i \in \{1,\dots,N\}$ indexes the sample and $j \in \{1,\dots,D\}$ indexes the feature column. For a predicted feature value $\hat{x}_{i,j}$ and its underlying ground truth $x_{i,j}$, the localized reconstruction loss is defined as:
\begin{equation}
    \ell_{feat}^{(i,j)} = \begin{cases} 
        \frac{1}{2}(\hat{x}_{i,j} - x_{i,j})^2 & \text{for } |\hat{x}_{i,j} - x_{i,j}| \le \delta, \\
        \delta |\hat{x}_{i,j} - x_{i,j}| - \frac{1}{2}\delta^2 & \text{otherwise},
    \end{cases}
    \label{eq:huber_loss}
\end{equation}
where $\delta > 0$ is a strict threshold hyperparameter governing the transition regime between squared and absolute penalty scales.

To further stabilize the optimization space across heterogeneous downstream semantics, we enforce a dynamic loss balancing strategy. Because each synthesized multi-task batch contains radically fluctuating proportions of classification samples, regression samples, and masked features, assigning static scalar weights inevitably results in gradient domination or representation collapse. To counteract this, the overall loss aggregates each task-specific objective averaged strictly over its valid samples. 

We define $\mathcal{B}_{cls}$ and $\mathcal{B}_{reg}$ as the varying sets of sample indices $i$ natively apportioned to classification and regression tasks in the current batch. The dynamically balanced total loss objective is aggregated as:
\vspace{-2mm}
\begin{equation}
    \begin{split}
        \mathcal{L}_{total} =\; & \mathbb{I}_{cls} \left( \frac{1}{|\mathcal{B}_{cls}|} \sum_{i \in \mathcal{B}_{cls}} \ell_{ce}^{(i)} \right) + 
        \mathbb{I}_{reg} \left( \frac{1}{|\mathcal{B}_{reg}|} \sum_{i \in \mathcal{B}_{reg}} \ell_{reg}^{(i)} \right) \\
        & + \mathbb{I}_{mask} \left( \frac{1}{|\mathcal{M}|} \sum_{(i,j) \in \mathcal{M}} \ell_{feat}^{(i,j)} \right),
    \end{split}
    \label{eq:dynamic_loss}
\end{equation}
where $\ell_{ce}^{(i)}$ calculates the standard cross-entropy penalty given categorical probabilities, and $\ell_{reg}^{(i)}$ leverages a task-level Huber objective homologous to Eq.~\ref{eq:huber_loss} for stable continuous target optimization. Functionally, $|\mathcal{B}_{cls}|$, $|\mathcal{B}_{reg}|$, and $|\mathcal{M}|$ express the absolute cardinality of their specific valid domains. The indicator functions $\mathbb{I}_{cls}, \mathbb{I}_{reg}, \mathbb{I}_{mask} \in \{0, 1\}$ operate as computational logical switches  activating specific loss modules exclusively when their pertinent sets evaluate as non-empty during the stochastic sampling batch.

Finally, to preclude numeric saturation endemic to natively sporadic structured templates, precise boolean verification arrays explicitly mask any missing data cells ahead of arithmetic projection. This protocol comprehensively ensures that intrinsic `NaN` or `Inf` artifacts never taint the gradient trajectories traversing throughout arbitrary length resolutions.
The overall pre-training pipeline of \texttt{FEAT} is summarized in Algorithm~\ref{alg:feat_pretraining}

\begin{algorithm}[t]
\caption{Pre-training of FEAT}\label{alg:feat_pretraining}
\DontPrintSemicolon
\SetKwInOut{Input}{Input}
\SetKwInOut{Output}{Output}
\SetKw{Return}{return}
\SetKwComment{tcp}{$\triangleright$ }{}
% 【关键修改1：强制将 ForEach 渲染为分开的 for each】
\SetKwFor{ForEach}{for each}{do}{end} 

\Input{Hybrid generator $\mathcal{G}$; Model $f_{\theta}$ with $L$ dual-axis blocks; Steps $T$; Context ratio $r_{ctx}$}
\Output{Pre-trained parameters $\theta^{*}$}

\BlankLine
% 【关键修改2：在初始化部分用极简的语言定义指示函数】
Initialize $\theta$ randomly; define indicator function $\mathbb{I}_{\{\cdot\}} \in \{0, 1\}$\; 
Setup AdamW optimizer with gradient clipping\;
\For{$t \leftarrow 1$ \KwTo $T$}{
    \BlankLine
    Sample dimensions $(N, D)$ and task $\mathcal{T} \in \{\text{CLS}, \text{REG}\}$\;
    $(\mathbf{X}, \mathbf{y}) \leftarrow \mathcal{G}(N, D, \mathcal{T})$\;
    Split context $(\mathbf{X}_L, \mathbf{y}_L) \leftarrow (\mathbf{X}_{1:N_L}, \mathbf{y}_{1:N_L})$, where $N_L = \lfloor r_{ctx} \cdot N \rfloor$\;
    Target set $\mathbf{X}_U \leftarrow \mathbf{X}_{N_L+1:N}$; sample mask ratio $\rho \sim \mathcal{U}(\rho_{\min}, \rho_{\max})$\;
    Generate mask $\mathcal{M}$ on $\mathbf{X}_U$ with ratio $\rho$, replacing cells $(i,j) \in \mathcal{M}$ with zeros\;

    \BlankLine
    
    \ForEach{$(i, j) \in [N] \times [D]$}{
        % 这里完美复用上面定义的指示函数
        $\mathbf{e}^{val}_{i,j} \leftarrow \mathbb{I}_{\{(i,j) \in \mathcal{M}\}} \mathbf{m}_{mask} + \mathbb{I}_{\{(i,j) \notin \mathcal{M}\}} \text{MLP}(x_{i,j})$\;
    }
    Sample orthogonal $\mathbf{O} \in \mathbb{R}^{D \times \frac{d}{4}}$ and compute $\mathbf{loc}^{col}_j \leftarrow \mathbf{o}_j \mathbf{W}_{dfe}$\;
    $\mathcal{X}^{(0)}_{i,j} \leftarrow \text{LayerNorm}(\mathbf{e}^{val}_{i,j} + \mathbf{loc}^{col}_j)$ yielding $\mathcal{X}^{(0)} \in \mathbb{R}^{N \times (D+1) \times d}$\;

    \BlankLine
    
    \For{$l \leftarrow 1$ \KwTo $L$}{
        $\mathcal{H}^{(l)} \leftarrow \text{MHSA-FFN}(\mathcal{X}^{(l-1)})$ \tcp*{Feature axis}
        $\mathcal{X}^{(l)} \leftarrow \text{AFBM-or-ConvGLA}(\mathcal{H}^{(l)})$ \tcp*{Sample axis}
    }

    \BlankLine
    
    Extract task $\mathbf{z}^{y}_i \leftarrow \mathcal{X}^{(L)}_{i, D+1, :}$ and feature $\mathbf{z}^{x}_{i,j} \leftarrow \mathcal{X}^{(L)}_{i, j, :}$ for $i > N_L$\;
    Predict $\hat{\mathbf{y}}_i \leftarrow \text{TaskHead}(\mathbf{z}^{y}_i, \mathcal{T})$ and $\hat{x}_{i,j} \leftarrow \text{MVIH}(\mathbf{z}^{x}_{i,j})$\;
    $\mathcal{L}_{total} \leftarrow \mathcal{L}_{task}(\hat{\mathbf{y}}, \mathbf{y}_U) + \mathcal{L}_{feat}(\hat{\mathbf{x}}, \mathbf{X}_U[\mathcal{M}])$ \tcp*{Eq.~\ref{eq:dynamic_loss}}
    Update $\theta$ via back-propagation using $\nabla_{\theta}\mathcal{L}_{total}$\;
}
\BlankLine
\Return{$\theta^{*} \leftarrow \theta$}\;
\end{algorithm}

\section{Theoretical Analysis}
\label{sec:theory}

In this section, we present the theoretical analysis of \texttt{FEAT}, focusing on its linear scalability, numerical stability over long sequences, and global permutation-invariant contextualization. Due to space constraints, we provide proof sketches.
% Full rigorous proofs are available in our extended technical report [Citation/Link].

\subsection{Complexity and Scalability}
\label{subsec:theory_complexity}
\texttt{FEAT} achieves linear scalability through its sample-axis encoder, which processes data using sequential state-space modeling instead of pairwise interactions.
\vspace{-1mm}
\begin{theorem}[Linear Complexity]
\label{thm:complexity}
% Given an embedded input tensor $\mathcal{X} \in \mathbb{R}^{N \times D \times d}$ (where $N$ is the tuple count, $D$ is the feature count, and $d$ is the embedding dimension), the computational time complexity of the \texttt{FEAT} sample-axis encoder is strictly $\mathcal{O}(N)$.

Given an embedded input tensor $\mathcal{X} \in \mathbb{R}^{N \times D \times d}$, where $N$ is the number of tuples, $D$ is the number of features, and $d$ is the embedding dimension, the time complexity of the \texttt{FEAT} sample-axis encoder is $\mathcal{O}(N)$.
\end{theorem}

\vspace{-4mm}

\begin{proof}[Proof Sketch]
% The forward pass complexity is bounded by its sequential operations across $D$ feature streams. For the 3-layer AFBM, the discrete state transitions and input projections per token require $\mathcal{O}(d_{state})$ and $\mathcal{O}(d \cdot d_{state})$ respectively, where $d_{state}$ is the latent state dimension. A bidirectional pass scales as $\mathcal{O}(N \cdot D \cdot d \cdot d_{state})$. Subsequently, the Conv-GLA layer applies a 1D convolution requiring $\mathcal{O}(N \cdot D \cdot d \cdot K)$ operations (with kernel size $K$), followed by Gated Linear Attention state accumulation requiring $\mathcal{O}(N \cdot D \cdot d^2)$ FLOPs. Summing these terms yields a total operations count of $\mathcal{T}_{total} = \mathcal{O}\!\left(N \cdot D \cdot (d \cdot d_{state} + d \cdot K + d^2)\right)$. Since $D$, $d$, $d_{state}$, and $K$ are architectural constants independent of the database scale $N$, the asymptotic temporal complexity scales strictly as $\mathcal{O}(N)$. Spatial complexity similarly scales linearly as it only caches the localized running states.

The sample-axis encoder, composed of the AFBM and Conv-GLA modules, processes the input through sequential operations along the sample dimension.
(i) For the 3-layer AFBM, the discrete state transition and input projection at each token require $\mathcal{O}(d_{state})$ and $\mathcal{O}(d \cdot d_{state})$ operations, respectively, where $d_{state}$ denotes the latent state dimension. 
Across $N$ tuples and $D$ feature streams, the bidirectional AFBM therefore requires $\mathcal{O}(N \cdot D \cdot d \cdot d_{state})$ operations. 
(ii) The subsequent Conv-GLA layer consists of a local convolution and gated linear attention. 
The 1D convolution with kernel size $K$ requires $\mathcal{O}(N \cdot D \cdot d \cdot K)$ operations. 
The gated linear attention accumulates states in a recurrent form and requires $\mathcal{O}(N \cdot D \cdot d^2)$ operations.

Combining these components, the overall computational cost is proportional to $N \cdot D \cdot (d \cdot d_{state} + d \cdot K + d^2)$. 
Since $D$, $d$, $d_{state}$, and $K$ are architectural constants independent of the database size $N$, the total time complexity scales as $\mathcal{O}(N)$. 
The space complexity also scales linearly with $N$, as the encoder only maintains localized running states rather than constructing a pairwise attention matrix.
\end{proof}
\vspace{-3mm}

\subsection{Numerical Stability for Long Sequences}
\label{subsec:theory_variance}
\texttt{FEAT} ensures numerical stability over large datasets through the Conv-GLA module, which controls the accumulation of information and bounds the variance of the internal states.
\vspace{-1.5mm}
\begin{theorem}[Variance Boundedness]
\label{thm:variance}
% The Conv-GLA mechanism strictly bounds the variance of the accumulated state representation, mathematically avoiding the $\mathcal{O}(N)$ divergence (``linear trap'') inherent in standard linear attention as sequence length $N \to \infty$.

Given the output tensor $\mathcal{X}^{(l)} \in \mathbb{R}^{N \times D \times d}$ from the previous layer $l$, where $N$ is the number of samples, $D$ is the number of features, and $d$ is the embedding dimension, the Conv-GLA module takes $\mathcal{X}^{(l)}$ as input and produces an accumulated state representation whose variance remains bounded and does not grow with $N$.
\end{theorem}
\vspace{-5mm}
\begin{proof}[Proof Sketch]
% Assume uncorrelated input value vectors $\mathbf{v}_{i} \in \mathbb{R}^d$ with variance $\text{Var}(\mathbf{v}_i) = \sigma^2 \mathbf{I}$.
% % In naive linear attention, the accumulated state variance diverges linearly as $\sum_{i=1}^N \sigma^2 \mathbf{I} = \mathcal{O}(N)$. 
% In Conv-GLA, the 1D convolution weights $w_m$ structurally attenuate the step-wise variance. The smoothed vectors $\tilde{\mathbf{v}}_i$ have an expectation $\boldsymbol{\mu}$ and an attenuated variance bounded by $\sigma^2 \mathbf{I} \sum w_m^2$ (via the Cauchy-Schwarz inequality). 

% Crucially, the gating mechanism $g_{i} \in (0,1)$ acts as a learned information sieve. Assuming there are $M \ll N$ informative tokens, $g_i$ can be approximated as a Bernoulli variable with activation probability $\rho = M/N$. Applying the Law of Total Variance to the gated vector $\mathbf{z}_i = g_i \cdot \tilde{\mathbf{v}}_i$ yields $\text{Var}(\mathbf{z}_i) \approx \rho \text{Var}(\tilde{\mathbf{v}}_i) + \rho(1-\rho)\boldsymbol{\mu}\boldsymbol{\mu}^\top$. Summing over $N$ uncorrelated steps, the global variance is $\text{Var}(\mathbf{S}_{N}) = N \cdot \text{Var}(\mathbf{z}_i) \le M ( \sigma^2 \mathbf{I} \sum w_m^2 + \boldsymbol{\mu}\boldsymbol{\mu}^\top )$. Because the weights are normalized ($\sum w_m^2 \le 1$), the asymptotic variance is constrained exclusively by the constant $M$ intrinsic concepts, rendering it mathematically independent of arbitrary lengths $N$ and ensuring $\mathcal{O}(1)$ numerical stability.

Given an input tensor $\mathcal{X}^{(l)} \in \mathbb{R}^{N \times D \times d}$, 
we analyze an arbitrary feature dimension, since Conv-GLA operates along the sample axis independently for each feature. 
Let $\mathbf{v}_{i} \in \mathbb{R}^d$ denote the value vector of the $i$-th sample under this feature dimension. 
Assume $\mathbf{v}_{i}$ are uncorrelated with $\mathrm{Var}(\mathbf{v}_i)=\sigma^2\mathbf{I}$.

The Conv-GLA module first applies a 1D convolution along the sample axis. 
Let $w_m$ denote the convolution weights and $\tilde{\mathbf{v}}_i$ the smoothed value vector. 
By the Cauchy-Schwarz inequality, $\mathrm{Var}(\tilde{\mathbf{v}}_i) \le \sigma^2\mathbf{I}\sum_m w_m^2$. 
With normalized weights ($\sum_m w_m^2 \le 1$), the convolution does not amplify the variance.

The gated representation is defined as $\mathbf{z}_i = g_i \cdot \tilde{\mathbf{v}}_i$, where $g_i \in (0,1)$. 
Assume the gate activates an expected number of $M \ll N$ informative samples, i.e., $g_i$ follows a Bernoulli variable with probability $\rho = M/N$. 
By the Law of Total Variance, the variance of $\mathbf{z}_i$ can be expressed as:
\vspace{-1mm}
\begin{equation}
\mathrm{Var}(\mathbf{z}_i) \approx \rho \mathrm{Var}(\tilde{\mathbf{v}}_i) + \rho(1-\rho)\boldsymbol{\mu}\boldsymbol{\mu}^{\top},
\vspace{-1mm}
\end{equation}
where $\boldsymbol{\mu}=\mathbb{E}[\tilde{\mathbf{v}}_i]$.

Aggregating over $N$ samples yields $\mathrm{Var}(\mathbf{S}_{N}) = N \cdot \mathrm{Var}(\mathbf{z}_i)$. 
Substituting $\rho=M/N$, the accumulated variance satisfies:
\vspace{-1mm}
\begin{equation}
\mathrm{Var}(\mathbf{S}_{N}) \le M(\sigma^2\mathbf{I}\sum_m w_m^2 + \boldsymbol{\mu}\boldsymbol{\mu}^{\top})
\vspace{-1mm}
\end{equation}
Since $M$ and the model parameters are independent of $N$, the accumulated variance is bounded with respect to the number of samples.

\end{proof}

\vspace{-6.5mm}

\subsection{Permutation-Invariant Global Contextualization}
\label{subsec:theory_bifusion}
\texttt{FEAT} enables global, permutation-invariant contextualization through the AFBM module, which propagates information bidirectionally across all samples.
\vspace{-1.5mm}
\begin{theorem}[Zero-Phase Global Receptive Field]
\label{thm:bifusion}
% The AFBM module in \texttt{FEAT} guarantees a non-zero gradient influence Jacobian $\mathcal{I}_{fused}(i, k) > 0$ for all $i,k \in \{1, \dots, N\}, i\leq k$, mathematically eliminating the causal phase shift and providing a learned symmetric receptive field equivalent to full self-attention.

Given an input tensor $\mathcal{X}^{(l-1)} \in \mathbb{R}^{N \times D \times d}$, assume that the forward and backward Mamba state transitions are non-degenerate and the fusion projection assigns non-zero weights to both directional states. 
For any samples $i,k \in \{1,\dots,N\}$, let $\mathbf{h}_i$ denote the fused AFBM representation of the $i$-th sample and let $\mathcal{X}^{(l-1)}_k$ denote the previous-layer representation of the $k$-th sample. 
Then, the fused AFBM representation induces a non-zero gradient influence Jacobian
$\mathcal{I}_{fused}(i,k)=\left\|\frac{\partial \mathbf{h}_i}{\partial \mathcal{X}^{(l-1)}_k}\right\|_F>0$.

\end{theorem}
\vspace{-3.5mm}
\begin{proof}[Proof Sketch]
% Let $\mathcal{I}(i, k) = \left\| \frac{\partial \mathbf{h}_{i}}{\partial \mathcal{X}_{k}} \right\|_F$ denote the influence Jacobian of sample $k$ on state $i$. Standard causal SSMs suffer a representational deficit where $\mathcal{I}(i, k) = 0$ for all future inputs $k > i$. Our AFBM recursively unrolls the discretized Mamba-2 dynamics (with state transition matrix $\mathbf{\bar{A}}$ and input matrix $\mathbf{\bar{B}}$) into forward ($\overrightarrow{\mathbf{h}}$) and backward ($\overleftarrow{\mathbf{h}}$) sequences. The fused representation $\mathcal{X}_{i}^{(l)}$ is derived via learned affine projections $\mathbf{W}_{\to}$ and $\mathbf{W}_{\gets}$. 

% For any row distance $\Delta = |i - k|$, the gradient propagates through the forward stream (if $k < i$) or the backward stream (if $k > i$). Since tabular data instances are intrinsically permutation-invariant (i.e., row order is arbitrary), the network dynamics force the fusion weights to balance optimally during backpropagation: $\mathbf{W}_{\to} \approx \mathbf{W}_{\gets}$. Consequently, the directional influence becomes symmetric: $\left\| \frac{\partial \mathcal{X}_{i}^{(l)}}{\partial \mathcal{X}_{i-\Delta}} \right\|_F \approx \left\| \frac{\partial \mathcal{X}_{i}^{(l)}}{\partial \mathcal{X}_{i+\Delta}} \right\|_F \approx \left\| \mathbf{W}_{\to} \mathbf{\bar{A}}^\Delta \mathbf{\bar{B}} \right\|_F$. This proves that causal bias is structurally canceled, endowing the model with an $\mathcal{O}(N^2)$ global receptive field via $\mathcal{O}(N)$ sequential operations.

AFBM unrolls the discretized Mamba-2 dynamics along the sample axis in two directions. 
Let $\bar{\mathbf{A}}$ and $\bar{\mathbf{B}}$ denote the state transition matrix and input matrix, respectively. 
The forward and backward hidden sequences are denoted by $\overrightarrow{\mathbf{h}}$ and $\overleftarrow{\mathbf{h}}$. 
The fused representation of sample $i$ is computed as
\begin{equation}
    \mathbf{h}_i
=
\mathbf{W}_{\to}\overrightarrow{\mathbf{h}}_i
+
\mathbf{W}_{\gets}\overleftarrow{\mathbf{h}}_i .
\end{equation}

For any pair of samples $(i,k)$, let $\Delta=|i-k|$. 
When $k \leq i$, sample $k$ can influence sample $i$ through the forward Mamba stream. 
Under the non-degenerate state transition assumption, this path contributes a non-zero term proportional to
$\left\|\mathbf{W}_{\to}\bar{\mathbf{A}}^{\Delta}\bar{\mathbf{B}}\right\|_F$. 
When $k \geq i$, sample $k$ can influence sample $i$ through the backward Mamba stream, contributing a non-zero term proportional to
$\left\|\mathbf{W}_{\gets}\bar{\mathbf{A}}^{\Delta}\bar{\mathbf{B}}\right\|_F$. 

Since the fusion projection assigns non-zero weights to both directional states, the corresponding active path is preserved in the fused representation. 
Thus, for every pair $(i,k)$, at least one valid directional path contributes a non-zero gradient term. 
Therefore,
\begin{equation}
\mathcal{I}_{fused}(i,k)
=
\left\|
\frac{\partial \mathbf{h}_i}{\partial \mathcal{X}^{(l-1)}_k}
\right\|_F
>0,
\end{equation}
for all $i,k \in \{1,\dots,N\}$.

\end{proof}
\vspace{-5.5mm}
\vspace{-1mm}
\section{Experiment}
\label{sec:experiment}

To evaluate the proposed \texttt{FEAT}\footnote{Model files:
\url{https://drive.google.com/file/d/1QWiolpkFdR9lQ6uGzpuMfbiIanlHHdaY/view?usp=drive_link}} framework, our experiments are designed to answer two fundamental research questions (RQs):

\begin{itemize}[leftmargin=*,labelsep=4pt,itemsep=2pt,parsep=0pt]
    \item \textbf{RQ1 (Scalability \& Efficiency):} Does the \texttt{FEAT} empirically deliver the mathematically promised $\mathcal{O}(N)$ linear scaling, and how does it compare against the $\mathcal{O}(N^2)$ memory wall of standard sequence-based structured-data models?
    
    \item \textbf{RQ2 (Predictive Parity):} Can \texttt{FEAT} maintain competitive zero-shot predictive performance across diverse real-world datasets without succumbing to the representation collapse typical of linear sequence models?
\end{itemize}

\vspace{-3mm}

\subsection{Experimental Setup}

\textbf{Baselines.} We benchmark \texttt{FEAT} against a comprehensive suite of state-of-the-art structured-data learning models, spanning modern foundation models, automated machine learning (AutoML), and classical gradient boosting. From a feature-representation perspective, these baselines are categorized as follows:

\begin{itemize}[leftmargin=*,labelsep=4pt,itemsep=2pt,parsep=0pt]
    \item \textbf{LimiX} \citep{zhang2025limix} employs a full Transformer with asymmetric self-attention to model the joint distribution of structured features. 
    % While achieving upper-bound zero-shot performance, its full-attention mechanism imposes a quadratic $\mathcal{O}(N^2)$ memory bottleneck.
    
    \item \textbf{TabICL v2} \citep{qu2026tabiclv2betterfasterscalable} is an advanced ICL foundation model that handles massive contexts via a two-stage embedding pipeline. It first captures statistical properties through distribution-aware column-wise embeddings, followed by row-wise attention to model cross-feature dependencies.
    
    \item \textbf{TabPFN Suite} \citep{hollmann2025accurate, grinsztajn2025tabpfn, garg2025realtabpfn} employs Transformer-based prior-data fitted networks (PFNs) pre-trained on synthetic structural causal models for ICL. In our experiments, we evaluate TabPFN 2.5, which supports larger feature and context dimensions, and TabPFN-Real, which further benefits from continued pre-training on real-world datasets to improve feature generalization.
    
    \item \textbf{AutoGluon v1.5} \citep{erickson2020autogluon} is a state-of-the-art AutoML framework that automates robust feature engineering (e.g., scaling, imputation) and constructs multi-layer stack ensembles of diverse traditional and deep models, representing the upper bound of classical supervised learning.
    
    \item \textbf{CatBoost} \citep{prokhorenkova2018catboost} is a gradient boosting framework optimized for categorical features. It employs an ordered boosting scheme and target statistics to prevent target leakage and prediction shift during feature conversion.
    
    \item \textbf{LightGBM} \citep{ke2017lightgbm} is an efficient tree-boosting system using gradient-based one-side sampling (GOSS) and exclusive feature bundling (EFB) to handle high-dimensional and sparse feature spaces.
    
    \item \textbf{XGBoost} \citep{chen2016xgboost} is a scalable tree-boosting framework that employs sparsity-aware split finding to automatically handle missing feature values and uses a block structure for parallel computation.
\end{itemize}

\vspace{-2mm}

\textbf{Benchmark Datasets.} To rigorously evaluate predictive parity (RQ2) and ensure our model handles diverse structured regimes, we follow the robust evaluation protocol established in recent structured-data foundation studies \citep{zhang2025limix}. We utilize 12
heterogeneous datasets from the following standard benchmarking suites:

\vspace{-2mm}

\begin{itemize}[leftmargin=*]
    \item \textbf{TabPFN Suite} (\texttt{pfn\_cls}, \texttt{pfn\_reg}) \citep{hollmann2025accurate}: A rigorous suite of small-to-medium scale datasets designed to test zero-shot in-context generalization across unseen feature spaces.
    
    \item \textbf{Tabzilla Benchmark} (\texttt{Tabzilla-CLS}) \citep{mcelfresh2023neural}: A massive, highly diverse collection of structured-data classification tasks. It serves as a comprehensive stress test for algorithmic robustness, evaluating performance across wildly varying domain origins, feature complexities, and class imbalances.
    
    \item \textbf{TALENT Benchmark} 
    (\texttt{talent\_}\allowbreak\texttt{cls}, \allowbreak \texttt{talent\_}\allowbreak\texttt{reg}) 
    \citep{liu2025talent}: A comprehensive aggregation of diverse real-world structured-data tasks, a primary testbed for feature interaction robustness.

    \item \textbf{TabArena Benchmark} 
    (\texttt{tabarena\_}\allowbreak\texttt{cls\_}\allowbreak\texttt{benchmark}) 
    \citep{erickson2025tabarena}: A modern, highly challenging evaluation suite featuring complex structural schemas, class imbalances, and mixed data types.
    
    \item \textbf{LimiX Curated Suites} (\texttt{BCCO}, \texttt{CTR23}) \citep{zhang2025limix}: To evaluate scalability and domain shift, we adopt the specific dataset partitions curated by LimiX. This includes the \textbf{BCCO Suite} for high-cardinality signal-to-noise testing, \textbf{CTR23} for massive-scale extreme sparsity recommendation environments. 
    
    \item \textbf{GI Benchmark} (\texttt{gi\_cls}, \texttt{gi\_reg}): A proprietary, real-world industrial dataset suite sourced from internal production environments. It encompasses highly heterogeneous structured data spanning diverse sectors, including finance, gaming, and consumer goods. This benchmark serves as a rigorous testbed to evaluate the model's cross-industry generalization and robustness on complex, production-grade data.

    \item \textbf{Long-Sequence Benchmark} 
    (\texttt{long\_}\allowbreak\texttt{sequence\_}\allowbreak\texttt{cls\_}\allowbreak\texttt{benchmark}, \allowbreak \texttt{long\_}\allowbreak\texttt{sequence\_}\allowbreak\texttt{reg\_}\allowbreak\texttt{benchmark}): 
    A curated suite aggregating samples with extended sequence lengths from across all benchmarks to evaluate both classification and regression performance under long-context settings.
\end{itemize}

\textbf{Metrics.} Computational efficiency (RQ1) is evaluated via end-to-end inference latency (ms). Downstream predictive performance (RQ2) is measured using the area under the receiver operating characteristic curve (AUC), accuracy (ACC), and F1-score for classification tasks. For regression tasks, performance is evaluated using the root mean squared error (RMSE) and the coefficient of determination ($R^2$ Score). All inference latency measurements and benchmarking evaluations were executed on a dedicated server equipped with a single NVIDIA H200 GPU.

\vspace{-2mm}

\subsection{Scalability and Efficiency Evaluation}
\label{subsec:exp_efficiency}
\vspace{-4mm}
\begin{table}[H]
    \centering
    \caption{Inference Latency (ms), where the best performance is marked in \textbf{bold} and `-' denotes kernel error on a H200.}
    \vspace{-3mm}
    \label{tab:inference_latency}
    % \tabcolsep 控制列与列之间的空白间距。
    % 默认通常是 6pt，这里改成 4pt 以确保 5 列数据能完美塞进单栏且不改变字体大小。
    \setlength{\tabcolsep}{4pt} 
    \begin{tabular}{rrrrr}
        \toprule
        \textbf{Context} & \multicolumn{4}{c}{\textbf{Model}} \\
        \cmidrule(lr){2-5}
        \textbf{Size} & \textbf{FEAT} & \textbf{TabICL v2} & \textbf{LimiX} & \textbf{TabPFN} \\
        \midrule
        5,000       & \textbf{103.69}   & 205.27       & 499.57   & 1749.54 \\
        10,000      & \textbf{199.49}   & 509.38       & 687.34   & 2129.98 \\
        20,000      & \textbf{386.99}   & 970.12       & 1699.65  & 4099.78 \\
        50,000      & \textbf{949.80}   & 3402.48      & 8811.47  & 13326.41 \\
        100,000     & \textbf{1900.87}  & 9140.67      & -        & - \\
        200,000     & \textbf{3794.55}  & 30683.89     & -        & - \\
        300,000     & \textbf{5677.75}  & 66299.04     & -        & - \\
        400,000     & \textbf{7577.98}  & 117101.94    & -        & - \\
        500,000     & \textbf{9477.11}  & 182895.02    & -        & - \\
        700,000     & \textbf{13271.86} & 357657.08    & -        & - \\
        1,000,000   & \textbf{18964.73} & 731520.32    & -        & - \\
        1,300,000   & \textbf{24744.13} & 1236871.66   & -        & - \\
        \bottomrule
    \end{tabular}
\end{table}
\vspace{-3mm}
We evaluate the efficiency of \texttt{FEAT} by measuring inference latency under progressively increasing context sizes (from 5,000 up to 1,300,000), while fixing the feature dimension to $D=10$. We compare \texttt{FEAT} against TabICL v2, LimiX, and TabPFN to demonstrate its scalability under extremely large context windows. 
The results in Table~\ref{tab:inference_latency} illustrate that baseline models suffer severely from the quadratic complexity of full attention. LimiX and TabPFN exhibit high latency early on and fail completely due to memory bottlenecks (kernel errors) when the context reaches 100,000. While TabICL v2 avoids early termination, its latency explodes non-linearly; at 500,000 context, it takes over 182 seconds, and at 1.3 million, the inference time severely deteriorates to over 1200 seconds. 
In stark contrast, \texttt{FEAT} maintains a highly efficient, strictly linear latency growth. Even at the massive context size of 1.3 million, \texttt{FEAT} requires only 24.7 seconds to complete inference, achieving a nearly 50$\times$ speedup over TabICL v2. This massive efficiency gain confirms that relying on standard quadratic self-attention incurs prohibitive computation and memory costs for long contexts. \texttt{FEAT} effectively overcomes this bottleneck by performing cross-sample modeling with linear-complexity operators, enabling rapid and stable inference even at a million-scale context.

\vspace{-1mm}

\subsection{Predictive Performance Evaluation}
\label{subsec:exp_performance}

\begin{table*}[t]
\centering
\vspace{-2mm}
\caption{Classification benchmark results, where the best performance is marked in bold.}
\label{tab:cls_results}
\vspace{-4mm}
\resizebox{\textwidth}{!}{%
    \begin{tabular}{@{} lccc @{\hspace{6mm}} lccc @{\hspace{6mm}} lccc @{}}
    \toprule
    \textbf{Method} & \textbf{AUC} & \textbf{ACC} & \textbf{F1} & \textbf{Method} & \textbf{AUC} & \textbf{ACC} & \textbf{F1} & \textbf{Method} & \textbf{AUC} & \textbf{ACC} & \textbf{F1} \\
    \midrule
    \multicolumn{4}{@{}c@{\hspace{6mm}}}{\textbf{Long-sequence-CLS}} & \multicolumn{4}{@{}c@{\hspace{6mm}}}{\textbf{GI-CLS}} & \multicolumn{4}{@{}c@{}}{\textbf{Tabarena-CLS}} \\
    \cmidrule(r){1-4} \cmidrule(lr){5-8} \cmidrule(l){9-12}
    \texttt{FEAT}    & \textbf{0.9725} & \textbf{0.9551} & \textbf{0.9260} & \texttt{FEAT}    & \textbf{0.9063} & \textbf{0.8698} & \textbf{0.7541} & \texttt{FEAT}    & \textbf{0.8638} & \textbf{0.8808} & \textbf{0.7178} \\
    LimiX           & 0.9598 & 0.9281 & 0.8332 & LimiX           & 0.8810 & 0.8408 & 0.7009 & LimiX           & 0.8437 & 0.8680 & 0.5922 \\
    TabICL v2       & 0.9601 & 0.9201 & 0.8691 & TabICL v2       & 0.8945 & 0.8522 & 0.6640 & TabICL v2       & 0.8548 & 0.8750 & 0.7178 \\
    TabPFN 2.5      & 0.9594 & 0.9159 & 0.8673 & TabPFN 2.5      & 0.8966 & 0.8683 & 0.7053 & TabPFN 2.5      & 0.8519 & 0.8743 & 0.7127 \\
    TabPFN 2.5 Real & 0.9588 & 0.9149 & 0.8668 & TabPFN 2.5 Real & 0.8967 & 0.8682 & 0.7045 & TabPFN 2.5 Real & 0.8525 & 0.8735 & 0.7121 \\
    AutoGluon v1.5  & 0.9534 & 0.9205 & 0.8613 & AutoGluon v1.5  & 0.8534 & 0.8627 & 0.6657 & AutoGluon v1.5  & 0.8311 & 0.8690 & 0.7057 \\
    CatBoost        & 0.9524 & 0.9076 & 0.8540 & CatBoost        & 0.8826 & 0.8539 & 0.6568 & CatBoost        & 0.8415 & 0.8685 & 0.7074 \\
    LightGBM        & 0.9189 & 0.8978 & 0.8033 & LightGBM        & 0.8707 & 0.8506 & 0.6520 & LightGBM        & 0.8313 & 0.8652 & 0.7077 \\
    XGBoost         & 0.9510 & 0.9079 & 0.8544 & XGBoost         & 0.8433 & 0.8521 & 0.6576 & XGBoost         & 0.8267 & 0.8645 & 0.7125 \\
    \midrule
    \textbf{Method} & \textbf{AUC} & \textbf{ACC} & \textbf{F1} & \textbf{Method} & \textbf{AUC} & \textbf{ACC} & \textbf{F1} & \textbf{Method} & \textbf{AUC} & \textbf{ACC} & \textbf{F1} \\
    \midrule
    \multicolumn{4}{@{}c@{\hspace{6mm}}}{\textbf{Talent-CLS}} & \multicolumn{4}{@{}c@{\hspace{6mm}}}{\textbf{Tabzilla-CLS}} & \multicolumn{4}{@{}c@{}}{\textbf{BCCO-CLS}} \\
    \cmidrule(r){1-4} \cmidrule(lr){5-8} \cmidrule(l){9-12}
    \texttt{FEAT}    & \textbf{0.8998} & \textbf{0.8444} & \textbf{0.7825} & \texttt{FEAT}    & \textbf{0.9251} & \textbf{0.8973} & \textbf{0.8515} & \texttt{FEAT}    & 0.8579 & 0.7719 & 0.6916 \\
    LimiX           & 0.8958 & 0.8419 & 0.7434 & LimiX           & 0.9235 & 0.8864 & 0.8064 & LimiX           & \textbf{0.8628} & \textbf{0.7892} & 0.7056 \\
    TabICL v2       & 0.8930 & 0.8424 & 0.7760 & TabICL v2       & 0.9176 & 0.8852 & 0.8381 & TabICL v2       & 0.8593 & 0.7877 & \textbf{0.7172} \\
    TabPFN 2.5      & 0.8978 & 0.8401 & 0.7790 & TabPFN 2.5      & 0.9119 & 0.8842 & 0.8341 & TabPFN 2.5      & 0.8527 & 0.7797 & 0.6975 \\
    TabPFN 2.5 Real & 0.8968 & 0.8393 & 0.7816 & TabPFN 2.5 Real & 0.9096 & 0.8834 & 0.8329 & TabPFN 2.5 Real & 0.8527 & 0.7789 & 0.6958 \\
    AutoGluon v1.5  & 0.8855 & 0.8369 & 0.7663 & AutoGluon v1.5  & 0.9071 & 0.8720 & 0.8233 & AutoGluon v1.5  & 0.8319 & 0.7635 & 0.6843 \\
    CatBoost        & 0.8796 & 0.8276 & 0.7623 & CatBoost        & 0.9068 & 0.8657 & 0.8198 & CatBoost        & 0.8359 & 0.7660 & 0.6928 \\
    LightGBM        & 0.8733 & 0.8240 & 0.7570 & LightGBM        & 0.8910 & 0.8640 & 0.8137 & LightGBM        & 0.8243 & 0.7503 & 0.6793 \\
    XGBoost         & 0.8704 & 0.8219 & 0.7611 & XGBoost         & 0.9029 & 0.8667 & 0.8258 & XGBoost         & 0.8222 & 0.7531 & 0.6845 \\
    \bottomrule
    \end{tabular}%
}
\vspace{-3mm}
\end{table*}

\begin{table*}[t]
\vspace{-1mm}
\centering
\caption{Regression benchmark results, where the best performance is marked in bold.}
\label{tab:reg_results}
\vspace{-4mm}
\resizebox{\textwidth}{!}{%
    \begin{tabular}{@{} lcc @{\hspace{6mm}} lcc @{\hspace{6mm}} lcc @{}}
    \toprule
    \textbf{Method} & \textbf{RMSE $\downarrow$} & \textbf{R2 $\uparrow$} & \textbf{Method} & \textbf{RMSE $\downarrow$} & \textbf{R2 $\uparrow$} & \textbf{Method} & \textbf{RMSE $\downarrow$} & \textbf{R2 $\uparrow$} \\
    \midrule
    \multicolumn{3}{@{}c@{\hspace{6mm}}}{\textbf{Long-sequence-REG}} & \multicolumn{3}{@{}c@{\hspace{6mm}}}{\textbf{BCCO-REG}} & \multicolumn{3}{@{}c@{}}{\textbf{GI-REG}} \\
    \cmidrule(lr){1-3} \cmidrule(lr){4-6} \cmidrule(l){7-9}
   
    \texttt{FEAT}                 & \textbf{0.2585} & \textbf{0.8672} & \texttt{FEAT}                 & \textbf{0.3841} & \textbf{0.7950}  & \texttt{FEAT}                 & 0.4703 & \textbf{0.6732} \\
    LimiX                & 0.2798 & 0.8393 & LimiX                & 0.4073 & 0.7750 & LimiX               & \textbf{0.4464} & 0.6514 \\
    TabICL v2            & 0.2632 & 0.8502 & TabICL v2            & 0.3863 & 0.7926  & TabICL v2            & 0.4953 & 0.6395 \\
    TablePFN v2.5        & 0.2698 & 0.8399 & TablePFN v2.5        & 0.3856 & 0.7916 & TablePFN v2.5        & 0.4745 & 0.6708 \\
    TablePFN v2.5 (Real) & 0.2734 & 0.8341 & TablePFN v2.5 (Real) & 0.3861 & 0.7912  & TablePFN v2.5 (Real) & 0.4749 & 0.6683 \\
    AutoGluon v1.5       & 0.2635 & 0.8408 & AutoGluon v1.5       & 0.4034 & 0.7788 & AutoGluon v1.5       & 0.4639 & 0.6652 \\
    CatBoost             & 0.2794 & 0.8120 & CatBoost             & 0.4147 & 0.7685 & CatBoost             & 0.5122 & 0.5977 \\
    LightGBM             & 0.2874 & 0.7060 & LightGBM             & 0.4334 & 0.7557 & LightGBM            & 0.5655 & 0.4950 \\
   XGBoost              & 0.2811 & 0.8275 & XGBoost              & 0.4521 & 0.7255 & XGBoost              & 0.5760 & 0.4477 \\
    
    \midrule
    \multicolumn{3}{@{}c@{\hspace{6mm}}}{\textbf{Talent-REG}} & \multicolumn{3}{@{}c@{\hspace{6mm}}}{\textbf{CTR23-REG}} & \multicolumn{3}{@{}c@{}}{\textbf{PFN-REG}} \\
    \cmidrule(lr){1-3} \cmidrule(lr){4-6} \cmidrule(l){7-9}
    \texttt{FEAT}                 & 0.4708 & \textbf{0.6985} & \texttt{FEAT}                 & \textbf{0.4053} & \textbf{0.7608} & \texttt{FEAT}                 & \textbf{0.5257} & \textbf{0.6266} \\
    LimiX                & 0.4873 & 0.6790 & LimiX                & 0.4383 & 0.7370 & LimiX                & 0.5385 & 0.6169 \\
    TabICL v2            & 0.4717 & 0.6963 & TabICL v2            & 0.4182 & 0.7489 & TabICL v2            & 0.5277 & 0.6213 \\
    TablePFN v2.5        & 0.4635 & 0.6905 & TablePFN v2.5        & 0.4115 & 0.7530 & TablePFN v2.5        & 0.5296 & 0.6214 \\
    TablePFN v2.5 (Real) & \textbf{0.4633} & 0.6915 & TablePFN v2.5 (Real) & 0.4108 & 0.7488 & TablePFN v2.5 (Real) & 0.5264 & 0.6259 \\
    AutoGluon v1.5       & 0.4803 & 0.6741 & AutoGluon v1.5       & 0.4273 & 0.7346 & AutoGluon v1.5       & 0.5477 & 0.5999 \\
    CatBoost             & 0.4973 & 0.6685 & CatBoost             & 0.4473 & 0.7066 & CatBoost             & 0.5732 & 0.4965 \\
    LightGBM             & 0.5058 & 0.6537 & LightGBM             & 0.4695 & 0.6496 & LightGBM             & 0.5785 & 0.4869 \\
    XGBoost              & 0.5318 & 0.6168 & XGBoost              & 0.4883 & 0.6693 & XGBoost              & 0.5904 & 0.5471 \\
    \bottomrule
    \end{tabular}%
}
\vspace{-4mm}
\end{table*}
To evaluate the predictive efficacy of \texttt{FEAT}, we analyze its performance across 12 datasets in classification and regression benchmarks, as summarized in Table~\ref{tab:cls_results} and Table~\ref{tab:reg_results}. A prominent highlight is \texttt{FEAT}'s absolute superiority in long-context scenarios, achieving the highest performance on both the \textit{Long-sequence-CLS} and \textit{Long-sequence-REG} benchmarks. This significant margin over strong baselines stems directly from how different models handle large-scale data. Traditional attention-based models are inherently constrained by an $O(N^2)$ computational bottleneck, forcing them to rely on sampling strategies when faced with massive sample sizes. This inevitably creates a context-information bottleneck where the model only processes a truncated subset of the global data distribution. In contrast, \texttt{FEAT} leverages a linear complexity of $O(N)$, allowing it to directly ingest and process the entire dataset as a single global context. By enabling full-batch, end-to-end learning without the severe information loss, \texttt{FEAT} fully exploits the long-range dependencies across the entire sample space.

Beyond long-sequence tasks, \texttt{FEAT} also demonstrates exceptional consistency and competitiveness across standard and highly sparse benchmarks, securing the highest overall AUC of 0.9251 on Tabzilla-CLS and a leading RMSE of 0.4053 on CTR23-REG. While standard linear-attention or recurrent architectures often suffer from representation degradation—commonly referred to as the linear trap—when compressing extensive sequences, \texttt{FEAT} fundamentally circumvents this vulnerability. The strength of the \texttt{FEAT} model lies in its novel integration of bidirectional mamba-2 coupled with an explicit global memory mechanism. The bidirectional mamba-2 ensures that each sample dynamically attends to both preceding and succeeding feature dependencies, while the global memory acts as a stable structural anchor that explicitly limits variance accumulation during sequential processing. This synergistic design allows \texttt{FEAT} to preserve the rich, high-fidelity expressive power characteristic of quadratic-attention models, maintaining robust predictive performance across diverse data scales while completely eliminating the $O(N^2)$ computational overhead.

\vspace{-2mm}
\subsection{Performance on Missing Value Imputation}

To evaluate the ability to reconstruct incomplete structured data, we build a missing value imputation benchmark using datasets covering both classification and regression tasks: Tabzilla-CLS, tabarena\_cls\_benchmark, CTR23, and talent\_reg\_benchmark.
We introduce an overall missing rate of 20\% to the original complete tables. To simulate realistic missingness patterns, we adopt a hybrid masking strategy at three structural levels:
(i) \textit{Cell Masking (50\%)}: randomly masking individual entries;
(ii) \textit{Stripe Masking (30\%)}: masking contiguous features or samples to simulate row-wise or column-wise corruption;
(iii) \textit{Block Masking (20\%)}: masking 2D regions to model large-area missingness.
We compare our method with LimiX, the only existing large structured-data model that supports missing value imputation. 
Performance is evaluated using Normalized Root Mean Square Error (NRMSE) for numerical features and Accuracy for categorical features.

\vspace{-2mm}
% 注意这里使用了 [H] 来强制固定位置
\begin{table}[H]
\centering
\caption{Efficiency and performance comparison between \texttt{FEAT} and LimiX. The best results are marked in bold.}
\label{tab:efficiency_comparison}
% 使用 resizebox 将表格宽度限制为当前栏宽 (\columnwidth)
\vspace{-4mm}
\resizebox{\columnwidth}{!}{%
\begin{tabular}{@{} l c c c c @{}}
\toprule
\textbf{Model} & \textbf{Avg NRMSE} $\downarrow$ & \textbf{Avg Acc} $\uparrow$ & \textbf{Avg Time (s)} $\downarrow$ & \textbf{Total Time (s)} $\downarrow$ \\
\midrule
\texttt{FEAT}  & \textbf{0.1779} & \textbf{0.5567} & \textbf{1.6} & \textbf{492.9} \\
LimiX & 0.1825          & 0.5200          & 3.6          & 1148.1 \\
\bottomrule
\end{tabular}%
}
\vspace{-3mm}
\end{table}

Table~\ref{tab:efficiency_comparison} evaluates the performance-efficiency trade-off between \texttt{FEAT} and LimiX. \texttt{FEAT} achieves superior imputation quality improving average accuracy by 7.06\% and reducing NRMSE by 2.5\%, while significantly cutting computational overhead. This is because \texttt{FEAT} leverages a linear-complexity architecture for highly efficient global context modeling, while incorporating a robust denoising pre-training strategy that forces the explicit learning of deep structural dependencies rather than shallow local correlations, significantly enhancing imputation accuracy. These results confirm that \texttt{FEAT} effectively resolves the quadratic complexity bottleneck typical of attention-based models, delivering fast and accurate imputation without sacrificing model capacity.

\vspace{-2mm}
\subsection{Performance on Feature Relation Discovery}
\label{subsec:feature_relation}

% As a newly supported capability, \texttt{FEAT} extends beyond conventional predictive tasks to offer intrinsic feature relation discovery. While SFMs are typically treated entirely as black-box predictors, \texttt{FEAT} natively exposes its learned feature representations, empowering users to understand complex inter-column dynamics directly.

% Figure~\ref{fig:relation_heatmap} illustrates the feature relation matrix extracted from the model on the \textit{user\_knowledge} dataset. By aggregating the cross-feature attention weights layer-by-layer, we construct a $D \times D$ attention matrix (where $D$ denotes the number of features). Unlike traditional statistical metrics (e.g., Pearson or Spearman correlation coefficients) that primarily capture linear relationships, this attention-based heatmap unveils complex, non-linear dependencies, directional interactions, and latent structural effects among features. The color intensity of each cell strictly indicates the strength of the learned dependency from one feature to another.

% This novel built-in interpretability is highly beneficial for modern data analysis workflows. It allows database practitioners to directly utilize the foundation model for Exploratory Data Analysis (EDA)—effectively identifying redundant attributes, profiling schema topology, and extracting data-driven priors for downstream causal discovery pipelines—all executed via a single forward pass without requiring external correlation analysis tools. 

To evaluate the interpretability of \texttt{FEAT}, we test whether its learned feature representations can reveal inter-column relationships. 
Figure~\ref{fig:relation_heatmap} shows the feature relationship matrix extracted from \texttt{FEAT} on the \textit{user\_knowledge} dataset. 
We aggregate the cross-feature attention weights across layers to construct a $D \times D$ matrix, where $D$ denotes the number of features. 
Each cell indicates the learned dependency strength between two features.

Compared with traditional statistical measures such as Pearson and Spearman correlations, the attention-based matrix can capture more flexible feature dependencies, including non-linear and directional interactions. 
This provides a built-in tool for exploratory data analysis, helping users identify redundant attributes, understand schema-level relationships, and obtain data-driven signals for downstream analysis without relying on external correlation analysis tools.

\begin{figure}[htbp]
    \centering
    \vspace{-3mm}
    % Added the figure/ path so the compiler can find your PDF
    \includegraphics[width=0.85\columnwidth]{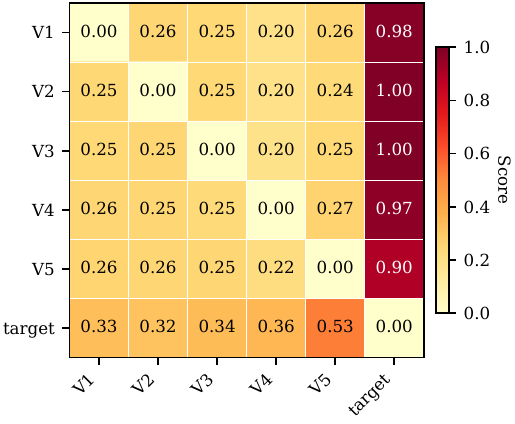}
    \vspace{-4mm}
    \caption{Feature relation matrix heatmap of \textit{user\_knowledge} dataset extracted from \texttt{FEAT}.}
    \label{fig:relation_heatmap}
    \vspace{-2mm}
\end{figure}

\vspace{-3mm}

\section{Conclusion}
\label{sec:conclusion}

% We propose \texttt{FEAT}, the first zero-shot tabular foundation model designed to achieve strict linear computational complexity ($\mathcal{O}(N)$) for massive structured datasets. It features a novel heterogeneous dual-axis architecture to overcome the representation collapse and causal biases typical of linear sequence models on non-sequential data. The framework incorporates an AFBM mechanism that enables dynamic, permutation-invariant noise filtration. Further, a Conv-GLA module provides a stable, infinite-capacity global memory to preserve long-range macroscopic feature dependencies. In addition, a heavy-tail-aware pre-training paradigm, driven by a hybrid Structural Causal Model, ensures robust optimization across highly variant empirical distributions. Extensive experiments across 11 diverse benchmark suites offer evidence that \texttt{FEAT} completely shatters the $\mathcal{O}(N^2)$ memory wall---expanding contextual capacity by $16\times$ and accelerating end-to-end inference---while matching or advancing the state-of-the-art zero-shot predictive performance of quadratic full-attention models. In future research, it is of interest to develop a multi-modal and more task of structured-data extension of \texttt{FEAT} and apply it in massive-scale industrial recommendation and financial forecasting scenarios.

In this work, we propose \texttt{FEAT}, a zero-shot foundation model designed for massive structured datasets with strict linear computational complexity $\mathcal{O}(N)$. \texttt{FEAT} is built on a heterogeneous dual-axis encoding architecture that integrates AFBM and Conv-GLA to perform efficient cross-sample modeling while preserving expressive representations. The AFBM component captures dynamic local dependencies in a permutation-invariant manner, whereas Conv-GLA introduces explicit global memory to maintain long-range interactions across large contexts. In addition, we design a heavy-tail-aware pre-training paradigm based on a hybrid SCM to ensure stable optimization under heterogeneous structured data distributions. Extensive experiments on 12 datasets from multiple classification and regression benchmarks demonstrate that \texttt{FEAT} achieves linear scalability, supports substantially larger sample contexts, and maintains competitive zero-shot predictive performance compared with strong full-attention baselines. In future work, we plan to extend \texttt{FEAT} to multi-modal structured-data settings and explore its applications in large-scale industrial scenarios such as recommendation systems and financial forecasting.

%\clearpage

\bibliographystyle{ACM-Reference-Format}
\bibliography{references}

\end{document}